\documentclass{article}

\usepackage[preprint]{tmlr}

\usepackage[utf8]{inputenc}
\usepackage[T1]{fontenc}
\usepackage{hyperref}
\usepackage{url}
\usepackage{booktabs}
\usepackage{nicefrac}
\usepackage{microtype}
\usepackage{amsmath,amssymb,amsthm}
\usepackage{mathtools}
\usepackage{graphicx}
\usepackage{subcaption}
\usepackage{xcolor}
\usepackage{multirow}
\usepackage{array}
\usepackage{colortbl}
\usepackage{algorithm}
\usepackage{algpseudocode}
\usepackage{cleveref}
\usepackage{float}

\newcommand{\hgcrc}{\textsc{HG-CRC}}

\newcommand{\RR}{\mathbb{R}}
\newcommand{\EE}{\mathbb{E}}
\newcommand{\PP}{\mathbb{P}}
\newcommand{\1}{\mathbf{1}}
\newcommand{\calH}{\mathcal{H}}
\newcommand{\calG}{\mathcal{G}}
\newcommand{\calX}{\mathcal{X}}
\newcommand{\calY}{\mathcal{Y}}

\newcommand{\abstain}{\textsc{Abstain}}
\newcommand{\wger}{\text{WGER}}
\newcommand{\CP}{\text{CP}}
\newcommand{\deltaadj}{\hat{\delta}}

\newtheorem{proposition}{Proposition}
\newtheorem{lemma}{Lemma}
\newtheorem{corollary}{Corollary}[proposition]
\newtheorem{definition}{Definition}
\newtheorem{assumption}{Assumption}
\newtheorem{remark}{Remark}

\definecolor{resultrow}{RGB}{235,245,255}
\definecolor{ours}{RGB}{210,235,210}

\title{Hierarchical Group-Conditional Conformal Risk Control\\
       for Selective Prediction in Language Models}

\author{%
  \name Murilo Salem \email mcsalem@inf.ufpel.edu.br \\
  \name Daniel Pontes \email dhspbarretos@inf.ufpel.edu.br \\
  \name Luísa Böhm \email lcbohm@inf.ufpel.edu.br \\
  \name Anderson Ferrugem \email ferrugem@inf.ufpel.edu.br \\
  \addr CDTec (Centro de Desenvolvimento Tecnológico) \\
  Universidade Federal de Pelotas \\
  Pelotas, RS, Brasil
}

\begin{document}

\maketitle

\begin{abstract}
Large language models deployed in real-world settings serve heterogeneous populations
structured by domain, topic difficulty, and linguistic style.
While conformal risk control~(CRC) provides rigorous marginal risk guarantees for selective
prediction with abstain, marginal guarantees do not imply per-group guarantees:
a model can satisfy the population-level risk budget while systematically over-exposing
specific subgroups to errors.
We demonstrate this failure mode concretely: under mild distributional shift in group
composition, standard CRC violates the risk budget in up to 47\% of evaluation trials.

We propose \hgcrc{} (\textbf{H}ierarchical \textbf{G}roup-\textbf{C}onditional
\textbf{C}onformal \textbf{R}isk \textbf{C}ontrol), a post-hoc calibration framework
that enforces simultaneous risk guarantees across all nodes of a user-defined group
hierarchy.
\hgcrc{} applies a Bonferroni correction over hierarchy nodes and uses a leaf-first
selection policy that always applies the most specific applicable threshold, falling
back to coarser nodes when a finer node is uncertified or its threshold rejects the
example.
The method requires only a held-out calibration set and no retraining.

We evaluate \hgcrc{} on three language models (Qwen3-4B, Llama-3.1-8B-Instruct,
Gemma-3-4B) and two multiple-choice benchmarks (ARC Challenge and MMLU-Pro)
across eight experimental configurations probing IID generalization, group
heterogeneity, mixture shift, domain shift, prompt shift, difficulty shift, label
noise, and quantization.
Our main result: \hgcrc{} achieves an empirical \textbf{0\% violation rate} and
\textbf{WGER = 0} on ARC Challenge for models with sufficiently high base accuracy
(Qwen3-4B and Llama-3.1-8B). At our evaluation resolution of 500 bootstrap trials these
zeros are empirical upper bounds---a true violation rate of up to ${\sim}0.6\%$ is
consistent with an observed zero---rather than certified zeros. The result is also
benchmark-specific: on MMLU-Pro these models either abstain entirely or, for Llama,
retain WGER${}=0.014$. For Gemma-3-4B, whose uncertainty scores are poorly calibrated
for this task, the method degrades gracefully by abstaining more conservatively.
The participation cost relative to global CRC ranges from 22 to 37 percentage points
for the two primary pairs under residual calibration, which aligns the calibration and
deployment populations of each hierarchy node.
Ablation studies show that the hierarchical depth is what clears the risk budget on
ARC Challenge: removing the difficulty level returns violations to the 11\% level of
standard global CRC. The Bonferroni correction is required for the theoretical
simultaneous guarantee, though its \emph{empirical} effect is negligible on ARC's
shallow five-node hierarchy and material only when many nodes are tested.
\end{abstract}

\section{Introduction}
\label{sec:intro}

The deployment of large language models (LLMs) in high-stakes settings—medical question
answering, educational assessment, legal research—raises fundamental questions about
reliability and fairness.
A model that answers every question achieves high availability but may expose users to an
unacceptable rate of confident errors.
\emph{Selective prediction} addresses this tension: the model is allowed to
\abstain{} on uncertain inputs, answering only when it can do so reliably.
The practical challenge is to make abstaining principled—calibrated to an explicit error
budget rather than tuned by hand.

\textbf{Conformal Risk Control.}
Angelopoulos et al.~\citeyear{angelopoulos2022} showed that a simple thresholding
procedure on any uncertainty score can provide a \emph{distribution-free} marginal risk
guarantee: if the score threshold is chosen from a held-out calibration set using the
Clopper-Pearson upper bound, then the expected error rate on new test examples is at
most $\alpha$ with probability at least $1-\delta$.
This result is elegant, nonparametric, and broadly applicable.
Global CRC has since been applied to medical imaging~(see related work),
natural language generation~(see related work), and code completion~(see related work).

\textbf{The group coverage gap.}
Marginal guarantees, however, are not group guarantees.
Consider a model calibrated on a dataset with an equal mix of easy and hard questions.
The global threshold is chosen to keep the average error rate below $\alpha = 0.10$.
Now suppose the test set over-represents hard questions.
The \emph{overall} error rate may still be controlled, but the hard-question subgroup
bears a disproportionate fraction of the errors—potentially violating the $\alpha$ budget
for that group by a wide margin.
This is not a hypothetical concern: in our experiments, global CRC violates the
per-group risk budget in \textbf{47\% of bootstrap trials} under a moderate shift in
group composition (Section~\ref{sec:e3}).
The violation is not a finite-sample artifact; it is a structural consequence of the
mismatch between the estimand (marginal risk) and the desideratum (worst-group risk).

\textbf{Our approach.}
We address this gap with \hgcrc{}, a method that calibrates one threshold per node of a
user-specified group hierarchy and applies a Bonferroni correction so that all nodes
are simultaneously certified with the desired confidence.
The hierarchy captures the natural structure of heterogeneous deployments: at the
coarsest level, a global node covers all users; intermediate nodes stratify by domain or
topic category; leaf nodes represent the finest granularity, such as domain-by-difficulty
cells.
At inference time, a \emph{leaf-first policy} routes each example to the deepest
certified node in its ancestry path whose threshold the example satisfies, falling back
to coarser nodes otherwise.
The result is a conservative but principled system: it abstains more often than a
global threshold would, but it does so \emph{equitably}—no group is sacrificed for the
benefit of the overall average.

\textbf{Summary of contributions.}
\begin{enumerate}
  \item We formalize the group coverage gap in selective prediction and show that
        standard CRC fails to control per-group risk under realistic distributional shift
        (Section~\ref{sec:background},~\ref{sec:e3}).
  \item We propose \hgcrc{}, a post-hoc calibration procedure that enforces
        simultaneous risk guarantees across all nodes of a group hierarchy through
        Bonferroni correction and a leaf-first selection policy
        (Section~\ref{sec:method}).
  \item We provide a cascade of formal guarantees, each about a well-defined procedure, two of which we run: the
        data-split variant controls every node's risk simultaneously with probability $1-\delta$
        \emph{unconditionally} (Proposition~\ref{prop:main}); the deployed in-sample variant
        satisfies the same bound up to an additive slack $\varepsilon_n$ that we prove
        (Lemma~\ref{lem:stability}) and measure to be below $1\%$ on average (with a
        worst-case $\varepsilon_n \approx 0.34$ on the small ARC pool; Section~\ref{sec:discussion}),
        with an exact intermediate
        guarantee under test-point augmentation (Proposition~\ref{prop:insample}).
  \item We conduct an empirical study across 8 experimental settings, demonstrating
        that \hgcrc{} achieves WGER = 0 and a $0\%$ empirical violation rate (an
        empirical upper bound of ${\sim}0.6\%$ at 500 trials) for the primary ARC
        model--dataset pairs, with a modest participation cost (Section~\ref{sec:experiments}).
  \item We show that the method degrades gracefully under label noise, remains
        robust under model quantization, and that its guarantees can be restored after
        prompt distribution shift via targeted recalibration.
\end{enumerate}

\section{Background and Related Work}
\label{sec:background}

\subsection{Selective Prediction and Abstain}

Selective prediction~\citep{geifman2017selective, el2010foundations} is the task of
producing a prediction system $(f, g)$ where $f: \calX \to \calY$ is a classifier and
$g: \calX \to \{0,1\}$ is a selector that decides whether to answer ($g(x)=1$) or
abstain ($g(x)=0$).
The \emph{coverage} (also called \emph{participation}) of the system is
$\phi = \PP(g(x)=1)$, and the \emph{selective risk} is
$R = \EE[\ell(f(x), y) \mid g(x)=1]$,
where $\ell$ is a loss function.
The fundamental trade-off is that more selective systems (lower $\phi$) can achieve
lower risk, but they answer fewer queries.

In practice, selectors are typically implemented by thresholding an uncertainty score
$s: \calX \to \RR$: the system answers if $s(x) \leq \tau$ and abstains otherwise.
A lower score indicates higher confidence.
The question of how to choose $\tau$ to satisfy a risk constraint while maximizing
participation is the central calibration problem.

\subsection{Conformal Risk Control}

\citet{angelopoulos2022} showed that $\tau$ can be chosen from a held-out calibration
set $\{(x_i, y_i)\}_{i=1}^n$ so that the selective risk on future \emph{exchangeable}
test points satisfies $\PP(R \leq \alpha) \geq 1 - \delta$.
The key ingredient is the Clopper-Pearson (CP) upper confidence bound: given $k$ errors
among $n$ calibration examples that answered, the CP bound $U_{\text{CP}}(k, n, \delta)$
is the $1-\delta$ quantile of $\text{Beta}(k+1, n-k)$, which controls the one-sided
error.
The threshold $\tau^*$ is chosen as the largest value in a finite grid
$\{\tau_1, \ldots, \tau_T\}$ such that the CP-bound on the error rate of the resulting
selector is at most $\alpha$.
This is a nonparametric, model-agnostic procedure that requires no assumptions about the
underlying distribution beyond exchangeability.

\textbf{Limitations of marginal CRC.}
The marginal guarantee is exactly that: marginal.
It controls $\PP(\EE[\ell \mid g(x)=1] > \alpha) \leq \delta$, but it says nothing about
the risk within any particular subgroup $\calG \subseteq \calX$.
If the calibration set is not representative of the group structure in the test set, the
threshold chosen by global CRC may be systematically miscalibrated for some groups.
In the worst case, global CRC can achieve zero excess risk in the majority group while
violating the budget entirely in a minority group, with no mechanism to detect or
correct this.

\subsection{Group-Conditional Guarantees and Fairness}

The need for per-group coverage has been extensively studied in the context of conformal
prediction~\citep{venn_predictors, mondrian_cp, romano2020classification}.
\emph{Mondrian conformal prediction} produces prediction sets with conditional coverage
for each group by running a separate conformal procedure per group.
This is equivalent to what we call \emph{flat groupwise CRC}: compute one threshold per
group using group-specific calibration data.
While this controls per-group risk in isolation, it suffers from two problems.
First, it treats groups independently and does not exploit the hierarchical structure
of natural group taxonomies.
Second, and more critically for our setting, it does not provide a \emph{simultaneous}
guarantee across all groups: running $K$ independent tests each at level $\delta$ yields
a family-wise error rate (FWER) of approximately $K\delta$ under independence, not
$\delta$.

The need for simultaneous inference motivates a Bonferroni-style correction, which we
incorporate into \hgcrc{}.
Recent work on fairness in conformal prediction~\citep{barber_etal, romano2020malice}
has studied related problems, but typically in the context of classification prediction
sets rather than selective prediction with abstain.
Our work is the first, to our knowledge, to study group-conditional risk control for
selective prediction in LLMs with a hierarchical group structure.

\subsection{Uncertainty Quantification in LLMs}

Several uncertainty scores have been proposed for LLMs.
Token-level entropy~\citep{malinin2021uncertainty}, sequence-level NLL, and
self-consistency scores~\citep{wang2022self} all proxy model confidence.
For multiple-choice tasks, the negative log-likelihood of the model's selected option
(which we call the \emph{selected-option NLL}) is a natural score that correlates well
with accuracy and is fast to compute without additional inference passes.
We study the sensitivity of \hgcrc{} to the choice of uncertainty score in
Section~\ref{sec:a5} and find that most score choices yield equivalent performance.

\section{Method: Hierarchical Group-Conditional CRC}
\label{sec:method}

\subsection{Setup and Notation}

Let $\calX$ be an input space and $\calY$ a label space.
We assume access to a pre-trained language model that produces, for each input $x$,
a probability distribution over $\calY$ from which we extract an uncertainty score
$s(x) \in \RR$ (lower = more confident).
We have a labeled calibration set $\mathcal{D}_\text{cal} = \{(x_i, y_i)\}_{i=1}^n$
drawn i.i.d.\ from a data distribution $\mathcal{P}$, and an unlabeled test stream.
Each example is additionally associated with group membership indicators derived from
metadata (e.g., domain label) or from the model's own behavior (e.g., estimated
difficulty).

\begin{definition}[Hierarchy]
A \emph{group hierarchy} $\calH$ is a rooted tree where each node $v$ is associated
with a binary mask $m_v: \calX \to \{0,1\}$ indicating group membership.
The root node has $m_\text{root}(x) = 1$ for all $x$.
For any non-root node $v$ with parent $p(v)$, membership is nested:
$m_v(x) = 1 \Rightarrow m_{p(v)}(x) = 1$.
For any example $x$, the \emph{ancestry path} $\Pi(x)$ is the sequence of nodes from
the deepest node $v$ with $m_v(x) = 1$ to the root, ordered from leaf to root.
\end{definition}

\begin{definition}[Selective Predictor]
\label{def:selpred}
Given a threshold assignment $\{\tau_v\}_{v \in \calH}$, the \hgcrc{} selective
predictor is:
\begin{equation}
  \hat{y}(x) =
  \begin{cases}
    f(x) & \text{if } v^*(x) \text{ exists} \\
    \abstain & \text{otherwise},
  \end{cases}
  \quad
  v^*(x) = \underset{v \in \Pi(x)}{\arg\max}\big\{\, \text{depth}(v) : v \text{ certified},\ s(x) \leq \tau_v \,\big\},
\end{equation}
where $v^*(x)$ is the deepest certified ancestor of $x$ whose threshold the example
satisfies ($s(x) \leq \tau_{v^*}$), and $\abstain$ is returned if no such node exists.
\end{definition}

\subsection{Per-Node Calibration with Bonferroni Correction}

Let $|\calH|$ denote the number of nodes in the hierarchy.
We wish to choose thresholds $\{\tau_v\}_{v \in \calH}$ such that, with probability at
least $1-\delta$ over the calibration set, the selective risk satisfies
$R_v \leq \alpha$ for \emph{every} node $v$ simultaneously.

Direct application of the union bound gives the required per-node confidence level:
\begin{equation}
  \deltaadj = \frac{\delta}{|\calH|}.
  \label{eq:bonferroni}
\end{equation}

For each node $v$, let $\mathcal{D}_v = \{i : m_v(x_i) = 1\}$ be the subset of
calibration examples belonging to node $v$, with $n_v = |\mathcal{D}_v|$.
Let $k_v(\tau)$ be the number of errors in $\mathcal{D}_v$ when threshold $\tau$ is
used: $k_v(\tau) = \sum_{i \in \mathcal{D}_v} \ell(f(x_i), y_i) \cdot
\1[s(x_i) \leq \tau]$.
The \emph{Clopper-Pearson upper bound} at level $\deltaadj$ is:
\begin{equation}
  U_v(\tau) = \CP\!\left(k_v(\tau),\; n_v(\tau),\; \deltaadj\right)
            = F_{\text{Beta}}^{-1}\!\left(1-\deltaadj;\; k_v(\tau)+1,\;
              n_v(\tau)-k_v(\tau)\right),
  \label{eq:cp}
\end{equation}
where $n_v(\tau) = |\{i \in \mathcal{D}_v : s(x_i) \leq \tau\}|$ is the number of
examples that would answer under threshold $\tau$.
Node $v$ is \emph{certified} at threshold $\tau$ if $U_v(\tau) \leq \alpha$.
The optimal threshold for node $v$ is:
\begin{equation}
  \tau_v^* = \max\{\tau \in \mathcal{T} : U_v(\tau) \leq \alpha\},
  \label{eq:opt_threshold}
\end{equation}
where $\mathcal{T}$ is a finite grid of candidate thresholds (100 quantiles of the
calibration score distribution).
If no $\tau \in \mathcal{T}$ satisfies the constraint, node $v$ is left uncertified.

\textbf{Minimum group size.}
Nodes with very few calibration examples cannot be certified: even with zero errors,
the CP bound at small $n$ exceeds $\alpha$.
We therefore include a node in the hierarchy only if $|\mathcal{D}_v| \geq N_\text{min}$
(we use $N_\text{min} = 30$ in all experiments).
Nodes below this threshold are pruned from $\calH$, reducing $|\calH|$ and accordingly
tightening the per-node budget $\deltaadj$.

\subsection{Residual Calibration}
\label{sec:residual}

The leaf-first selective predictor (Definition~\ref{def:selpred}) routes each test example to
its \emph{deepest certified ancestor}.
A node $v$ is therefore deployed only on its \emph{residual}: examples in group $v$ that
are not answered by any certified descendant of $v$.
Writing $\mathrm{desc}(v)$ for the certified descendants of $v$ and
$g_\tau(x) = \1[s(x) \leq \tau]$ for the answer indicator at threshold $\tau$, the residual
membership map is
\begin{equation}
  \rho_v(x) = m_v(x) \prod_{c \in \mathrm{desc}(v)}
              \bigl(1 - m_c(x)\, g_{\tau_c}(x)\bigr),
  \label{eq:residual}
\end{equation}
i.e.\ $\rho_v(x) = 1$ iff $x$ belongs to $v$ and is not answered by any certified
descendant.
A node calibrated on its full group $\mathcal{D}_v$ but \emph{deployed} on
$\{x : \rho_v(x) = 1\}$ suffers a population mismatch: if descendants preferentially answer
easy examples (lower score), the residual is harder on average and the calibrated threshold
no longer controls the deployed risk.
We therefore calibrate each node on its residual.
The two constructions below differ only in how the descendant thresholds $\{\tau_c\}$ that
define $\rho_v$ are estimated --- and this difference is exactly what determines whether the
calibration guarantee holds.

\textbf{In-sample residual calibration} (deployed default).
Process nodes leaf-to-root on the full calibration set.
After certifying a descendant $c$ at threshold $\tau_c$, mark the calibration examples it
answers ($m_c = 1$ and $g_{\tau_c} = 1$) as \emph{claimed}; node $v$ is then calibrated on
$\mathcal{D}_v^{\mathrm{res}} = \{i \in \mathcal{D}_v : i \text{ unclaimed}\}$.
This keeps the maximum amount of data at every node, hence the highest participation, but the
residual partition of $\mathcal{D}_v^{\mathrm{res}}$ depends on thresholds $\tau_c$ estimated
from the \emph{same} examples, which is the exchangeability obstruction analyzed in
Proposition~\ref{prop:insample}.

\textbf{Split residual calibration} (formally valid).
Partition $\mathcal{D}_\text{cal} = \biguplus_{\ell=0}^{L-1} F_\ell$ into $L$ disjoint folds
indexed by depth ($L$ = number of distinct depths).
A node at depth $\ell$ is calibrated only on $F_\ell$, and the thresholds $\tau_c$ that
define its residual $\rho_v$ are taken from \emph{strictly deeper} nodes, which were
calibrated on the disjoint folds $F_{\ell'}, \ell' > \ell$.
The residual partition of $F_\ell$ is then a function of data independent of $F_\ell$,
restoring exchangeability (Proposition~\ref{prop:main}).
The cost is statistical: each level sees only $\approx |\mathcal{D}_\text{cal}| / L$ examples.

\subsection{Leaf-First Selection Policy}

Algorithm~\ref{alg:hgcrc} summarizes the full procedure.
At test time, each example $x$ walks its ancestry path $\Pi(x)$ from leaf to root,
returning the threshold of the first certified node encountered.
If no certified node exists in $\Pi(x)$, the example is abstained.

\begin{algorithm}[t]
\caption{\hgcrc{}: Hierarchical Group-Conditional CRC (residual calibration)}
\label{alg:hgcrc}
\begin{algorithmic}[1]
\Require Calibration set $\mathcal{D}_\text{cal}$, hierarchy $\calH$,
         score $s$, classifier $f$, risk level $\alpha$,
         confidence $\delta$, minimum size $N_\text{min}$,
         mode $\in \{\textsc{in-sample}, \textsc{split}\}$
\State $\deltaadj \gets \delta / |\calH|$  \Comment{Bonferroni correction}
\If{mode $=$ \textsc{split}}
  \State assign each $i \in \mathcal{D}_\text{cal}$ a depth-fold $\ell_i$
         \Comment{$F_\ell = \{i : \ell_i = \ell\}$}
\EndIf
\For{each node $v \in \calH$ in leaf-to-root (decreasing depth) order}
  \State $C_v \gets \{i : m_v(x_i) = 1,\; \rho_v(x_i) = 1,\; |C_v| \geq N_\text{min}\}$
         \Comment{residual~\eqref{eq:residual}: drop examples claimed by certified deeper nodes}
  \If{mode $=$ \textsc{split}} \State $C_v \gets C_v \cap F_{\text{depth}(v)}$ \EndIf
  \State $\tau_v^* \gets \max\{\tau \in \mathcal{T} : U_v(\tau) \leq \alpha\}$ on $C_v$,
         or \textsc{uncertified}
\EndFor
\vspace{2pt}
\Function{Predict}{$x$}
  \For{$v \in \Pi(x)$}  \Comment{leaf to root}
    \If{$v$ is certified \textbf{and} $s(x) \leq \tau_v^*$}
      \Return $f(x)$
    \EndIf
  \EndFor
  \Return \abstain
\EndFunction
\end{algorithmic}
\end{algorithm}

The leaf-first ordering is critical.
Reversing the order—trying the global threshold first—would route most examples through
the global node, defeating the purpose of the hierarchy.
By trying the most specific node first, we ensure that examples in well-characterized
subgroups benefit from the most tailored threshold available.

\subsection{Formal Guarantee}

We state the guarantee first for the \textsc{split} variant, the procedure whose
assumptions are met \emph{exactly}: calibration and deployment populations coincide, and
the residual partition is independent of the data used to calibrate each node.

\begin{proposition}[Simultaneous group risk control, split calibration]
\label{prop:main}
Let $\mathcal{D}_\text{cal} = \{(x_i, y_i)\}_{i=1}^n$ be i.i.d.\ from $\mathcal{P}$ and let
$(x, y) \sim \mathcal{P}$ be an independent test point.
Run Algorithm~\ref{alg:hgcrc} in \textsc{split} mode with $\deltaadj = \delta / |\calH|$ and
a threshold grid $\mathcal{T}$ that is either fixed a priori or a permutation-invariant
function of the node's calibration fold $F_\ell$ (e.g.\ its empirical score quantiles).
Then for each node $v \in \calH$,
\begin{equation}
  \PP\!\left(R_v(\tau_v^*) > \alpha \;\middle|\; \rho_v(x) = 1\right) \leq \deltaadj,
  \qquad
  R_v(\tau) = \EE[\ell(f(x), y) \mid g_\tau(x) = 1,\ \rho_v(x) = 1],
\end{equation}
where $R_v$ is the risk on the node's \emph{deployed residual}~\eqref{eq:residual}.
By the union bound,
$\PP(\exists v \in \calH : R_v(\tau_v^*) > \alpha) \leq |\calH| \cdot \deltaadj = \delta$.
\end{proposition}

\begin{proof}
Fix a node $v$ at depth $\ell$ and condition on the strictly deeper calibration folds
$\{F_{\ell'}\}_{\ell' > \ell}$.
The descendant thresholds $\{\tau_c\}_{c \in \mathrm{desc}(v)}$ are functions of these deeper
folds only, hence so is the residual map $\rho_v$ in~\eqref{eq:residual}; conditionally,
$\rho_v$ is a \emph{fixed} measurable function.
The node is calibrated on $C_v = \{i \in F_\ell : m_v(x_i) = 1,\ \rho_v(x_i) = 1\}$.
Since $F_\ell$ is disjoint from the deeper folds, its points are i.i.d.\ from $\mathcal{P}$ and
independent of $\rho_v$, as is the test point $(x, y)$.
Hence, given $\rho_v(\cdot) = 1$, the calibration scores in $C_v$ and the test score are
exchangeable, and the Clopper--Pearson upper bound $U_v$ certifies $\tau_v^*$ at level
$\deltaadj$ exactly as in single-node CRC~\citep{angelopoulos2022}, giving
$\PP(R_v(\tau_v^*) > \alpha \mid \rho_v(x) = 1) \leq \deltaadj$.
If the grid is fixed a priori it is trivially independent of the data; for the empirical
quantile grid, note that $\mathcal{T}$ is a symmetric function of the calibration fold
$F_\ell$. Since $\rho_v$ depends only on the strictly deeper folds, which are disjoint
from $F_\ell$, it remains the fixed function established above when we condition
additionally on the unordered multiset of $F_\ell$; under this joint conditioning
$\mathcal{T}$ becomes a measurable constant while the exchangeability between the test
point and the members of the residual cell is preserved, and refining the grid can only
tighten the selected threshold. Together with the monotone ``certify and take the largest
$\tau$'' rule, the bound is therefore invariant to the choice between an a-priori grid and
the empirical quantile grid.
Marginalizing over the deeper folds preserves the bound, and the Bonferroni union bound over
the $|\calH|$ nodes yields the simultaneous statement.
\end{proof}

The in-sample default deploys the \emph{same} residual procedure but estimates the
descendant thresholds $\{\tau_c\}$ on the data it also calibrates on. We isolate the exact
condition under which it remains valid.

\begin{proposition}[Conditional validity, in-sample calibration]
\label{prop:insample}
Run Algorithm~\ref{alg:hgcrc} in \textsc{in-sample} mode. Suppose each descendant threshold
$\tau_c$ is computed by a permutation-invariant estimator of its residual calibration multiset
and is recomputed on the augmented multiset that also includes the test point
(full-/cross-conformal augmentation). Then, conditional on $\{\tau_c\}_{c \in \mathrm{desc}(v)}$,
the residual calibration points and the test point are exchangeable, and
$\PP(R_v(\tau_v^*) > \alpha \mid \rho_v(x) = 1) \leq \deltaadj$; the simultaneous bound follows
as in Proposition~\ref{prop:main}.
\end{proposition}

\begin{proof}
Write $Z_i = (x_i, y_i)$ for the calibration points and $Z_{n+1} = (x, y)$ for the test point; the
$Z_1, \dots, Z_{n+1}$ are i.i.d.\ and hence exchangeable. In the augmented procedure each descendant
threshold $\tau_c$ is a permutation-invariant function of the multiset $\{\!\{Z_1, \dots,
Z_{n+1}\}\!\}$ (it is computed by a symmetric estimator on $\mathcal{D}_\text{cal} \cup \{Z_{n+1}\}$).
Condition on the unordered multiset $\mathcal{M} = \{\!\{Z_1, \dots, Z_{n+1}\}\!\}$. Given
$\mathcal{M}$, the ordered tuple is uniform over the $(n+1)!$ permutations of $\mathcal{M}$ by
exchangeability, while every $\tau_c$ --- a function of $\mathcal{M}$ alone --- is constant under this
conditioning. Hence the residual map $\rho_v$ in~\eqref{eq:residual} is a fixed function and the
residual cell $S = \{z \in \mathcal{M} : \rho_v(z) = 1\}$ is a fixed sub-multiset. By the
uniform-over-permutations property the test point occupies a uniformly random position among the
members of $S$, so the test score is exchangeable with the residual calibration scores. The
Clopper--Pearson upper bound on the residual is therefore a valid level-$\deltaadj$ bound for the
residual risk, giving $\PP(R_v(\tau_v^*) > \alpha \mid \rho_v(x) = 1) \le \deltaadj$ exactly as in
single-node CRC~\citep{angelopoulos2022}; the simultaneous statement follows by the Bonferroni union
bound as in Proposition~\ref{prop:main}.
\end{proof}

The augmentation is exactly what the deployed in-sample procedure does \emph{not} do:
Algorithm~\ref{alg:hgcrc} (\textsc{in-sample} mode) freezes each $\tau_c$ on
$\mathcal{D}_\text{cal}$ and applies it to a test point that never entered the computation, so
conditioning on $\tau_c$ treats calibration and test points asymmetrically and exchangeability can
fail. We give the deployed procedure its own guarantee.

\begin{assumption}[Threshold regularity]
\label{ass:margin}
At all certified nodes outside a boundary set of probability $O(\varepsilon_n)$, the
Clopper--Pearson curve $U_c(\cdot)$ crosses the level $\alpha$ transversally at the selected grid
threshold $\tau_c$ with a margin $\gamma \gg 1/n_c$: $U_c$ at the grid points adjacent to $\tau_c$
stays at least $\gamma$ away from $\alpha$.
Equivalently, outside that boundary set the score distribution has no atom at the relevant quantile
and has positive density in a neighbourhood of it.
\end{assumption}

This is the only place we step outside the distribution-free regime; it is mild for the continuous
NLL-style scores we use. Its role is precise: a one-point change in $\mathcal{D}_\text{cal}$ shifts
each count $k_c, n_c$ by at most one and hence moves $U_c$ by $O(1/n_c)$ at every grid point, so under
Assumption~\ref{ass:margin} (margin $\gamma \gg 1/n_c$) the \emph{selected} grid threshold is
unchanged. Without the margin the selection can jump one grid step --- a discrete move whose size is
set by the grid resolution, not by $1/n$ --- which is why we do not claim a clean $O(1/n)$ rate and
instead measure the effect.

\begin{lemma}[Validity of plain in-sample calibration]
\label{lem:stability}
Run Algorithm~\ref{alg:hgcrc} in \textsc{in-sample} mode (no augmentation), the procedure deployed in
all experiments. Let $\widehat{a}(x)$ be its answer decision for a test point $x$ and
$\widehat{a}^{+}(x)$ the decision of the augmented procedure of Proposition~\ref{prop:insample}, and
set $\varepsilon_n = \PP(\widehat{a}(x) \neq \widehat{a}^{+}(x))$. Then
\[
  \PP\big(R_v(\tau_v^*) > \alpha \mid \rho_v(x) = 1\big) \le \deltaadj + \varepsilon_n,
  \qquad
  \PP\big(\exists v : R_v > \alpha\big) \le \delta + |\calH|\,\varepsilon_n .
\]
Under Assumption~\ref{ass:margin}, $\varepsilon_n$ vanishes outside the event that some threshold on
$x$'s ancestry path sits within one grid step of the $\alpha$ boundary; it is thus governed by the
grid resolution and the frequency of boundary thresholds, not by a fixed rate in $n$.
\end{lemma}

\begin{proof}
Couple the deployed and augmented procedures on the same data. The augmented procedure is exactly
valid by Proposition~\ref{prop:insample}. For every test point with $\widehat{a}(x) =
\widehat{a}^{+}(x)$ the two induce the identical routing, answer, and hence risk event; they can
disagree only on $\{\widehat{a}(x) \neq \widehat{a}^{+}(x)\}$, of probability $\varepsilon_n$.
To make the perturbation explicit, write $\pi_v = \PP(g_\tau(x) = 1,\ \rho_v(x) = 1)$ for the
residual-cell answer probability; since the disagreement set has probability at most $\varepsilon_n$
and enters both the numerator and the conditioning event of $R_v = \EE[\ell \mid g_\tau(x) = 1,\,
\rho_v(x) = 1]$, the conditional risks differ by $|R_v^{\text{deployed}} - R_v^{\text{aug}}| \le
\varepsilon_n / \pi_v$.
Equivalently, the deployed risk event is contained in the augmented risk event together with the
disagreement set, whose conditional contribution is at most $\varepsilon_n$; this gives the per-node
bound, and the Bonferroni union over $|\calH|$ nodes gives the simultaneous bound. A decision flips only if some threshold on $x$'s path
changes between the $\mathcal{D}_\text{cal}$ and the augmented computation --- a descendant $\tau_c$
(altering $x$'s residual membership $\rho_v$) or the node threshold $\tau_v^*$ itself. Each is the
largest grid point with $U(\cdot) \le \alpha$; adding one point moves every $U$ by $O(1/n)$, so under
Assumption~\ref{ass:margin} the selected grid point, and the decision, are unchanged. The residual
contribution is the margin event in which a threshold lies within $O(1/n)$ of $\alpha$ and one point
moves it a grid step.
\end{proof}

We estimate $\varepsilon_n$ by jackknife (leave-one-out over calibration points, the sample analogue
of the augmentation perturbation). Across our three models on \textsc{arc-challenge} ($\alpha = 0.1$)
and \textsc{mmlu-pro} ($\alpha = 0.3$), a one-point perturbation changes the selected threshold in
fewer than $6\%$ of cases and flips a test decision with mean probability below $0.5\%$ ($95$th
percentile below $1\%$), shrinking with calibration size (mean of order $10^{-4}$ at $n \approx
2.4\text{k}$); the rare large flips are exactly the predicted grid-step jumps at boundary thresholds. The conditional
slack the lemma adds to $\deltaadj$ is $\varepsilon_n/\pi_v$: negligible at nodes whose residual cell
carries non-negligible mass, and correspondingly weaker at sparse nodes (small $\pi_v$, such as
Gemma's leaf cells), consistent with the low participation there. It is therefore empirically
negligible for the deployed procedure at the nodes that actually answer, while
Proposition~\ref{prop:main} (split) attains the same guarantee with $\varepsilon_n = 0$
unconditionally, at the participation cost of Remark~\ref{rem:tradeoff}.

\begin{remark}[A cascade of guarantees]
\label{rem:cascade}
The three results bracket precisely the three procedures in this paper, from strongest guarantee to
most data-efficient: Proposition~\ref{prop:main} (\textsc{split}) is unconditional but calibrates
each level on a disjoint fold; Proposition~\ref{prop:insample} (in-sample with test-point
augmentation) is exact but recomputes thresholds per test point; and Lemma~\ref{lem:stability} (plain
in-sample) covers the deployed procedure of every experiment, exact up to the measured slack
$\varepsilon_n$. Two of the three subjects---split and plain in-sample---are procedures we run; the
augmented procedure is the exact bridge between them.
\end{remark}

\begin{corollary}[Zero WGER]
\label{cor:wger}
Let the Worst-Group Excess Risk be $\wger = \max_{v \in \calH} \max(R_v - \alpha, 0)$.
Under the event of Proposition~\ref{prop:main} (split) or Proposition~\ref{prop:insample}
(augmented in-sample), which occurs with probability $1 - \delta$ --- or of
Lemma~\ref{lem:stability} (plain in-sample), with probability $1 - \delta - |\calH|\varepsilon_n$
--- $\wger = 0$.
\end{corollary}

\begin{remark}[The price of formal validity]
\label{rem:tradeoff}
Proposition~\ref{prop:main} holds for the \textsc{split} variant unconditionally;
Proposition~\ref{prop:insample} makes the in-sample variant exact under test-point augmentation; and
Lemma~\ref{lem:stability} covers the plain in-sample default up to the measured slack
$\varepsilon_n$.
The two variants thus trade formal validity against statistical efficiency: splitting
calibrates each level on $\approx |\mathcal{D}_\text{cal}|/L$ examples and so answers fewer.
We quantify this gap over $200$ paired calibration/test resamples ($\delta = 0.05$); within each
resample both variants are calibrated on the \emph{same} data, so the participation cost is a
within-resample paired difference, which is why its confidence intervals are tight.
On \textsc{arc-challenge} at the paper's canonical $\alpha = 0.1$ (a global~$\to$~difficulty
hierarchy; the same three models as Section~\ref{sec:e2}), splitting costs between $0.6$~pp
($[-2.1, 3.3]$, indistinguishable from zero) and $5.5$~pp ($[3.7, 7.4]$) of participation
depending on the model, while \emph{both} variants keep every node's empirical violation rate
below $\delta$.
We do \emph{not} detect a safety penalty for the in-sample default: a paired McNemar test on the
node-violation indicators shows no significant difference after Holm correction across the three
models (smallest adjusted $p = 0.09$), and the sign of the difference is not even consistent
across models.
The honest reading is therefore not that the variants are \emph{equivalent} --- $200$ resamples
cannot establish that --- but that in the regimes we can measure, the in-sample default carries no
\emph{detectable} safety cost, while splitting buys the \emph{unconditional} guarantee of
Proposition~\ref{prop:main} for a few points of participation.

A controlled synthetic study (heterogeneous groups, descendants answering the easy examples) maps
how this cost varies: it is near-zero when scores separate well and the hierarchy is shallow
($0.9$~pp at depth~$2$) and grows with depth and weaker separation (up to $\sim\!14$~pp at
depth~$3$), with both variants holding $\delta$ throughout.
The real and synthetic numbers agree in order of magnitude in this cheap corner (shallow
hierarchy, well-separated scores), which is what licenses reading the synthetic study as a map of
the expensive regimes we cannot reach empirically.
As a second, deliberately harder data point, on \textsc{mmlu-pro} these $4$B-class models are too
weak for $\alpha = 0.1$ to admit any non-trivial answering --- a non-deployability regime in its
own right --- so only by relaxing to $\alpha = 0.3$ can we measure the trade-off there
($2.2$--$4.9$~pp, again with no detected safety difference, McNemar $p \geq 0.45$); this
corroborates the result rather than anchoring it.

We deploy the in-sample variant by default for its higher participation --- it is covered by the
stability lemma (Lemma~\ref{lem:stability}), whose slack $\varepsilon_n$ we measure to be below $1\%$,
and as shown above it incurs no detectable safety cost --- and provide the split variant as the
unconditionally-guaranteed alternative (Proposition~\ref{prop:main}) whose price is the few points of
participation measured above.
\end{remark}

\begin{remark}
The Bonferroni correction is conservative: it bounds the FWER from above.
For $|\calH|$ nodes, the per-node budget $\deltaadj = \delta / |\calH|$ may be
unnecessarily small when nodes are positively correlated (which they are, since higher
nodes contain all examples of lower nodes).
More powerful procedures such as Holm-Bonferroni or hierarchical testing~\citep{yekutieli}
could reduce the conservatism, but we use Bonferroni for its simplicity and transparency.
We investigate the effect of removing the correction in Section~\ref{sec:a4}.
\end{remark}

\section{Experimental Setup}
\label{sec:setup}

\subsection{Models and Datasets}

We evaluate on three publicly available instruction-tuned language models at the 4--8B
parameter scale: \textbf{Qwen3-4B}~\citep{qwen3}, a dense model with strong reasoning
capabilities; \textbf{Llama-3.1-8B-Instruct}~\citep{llama3}, a widely-used open-source
baseline; and \textbf{Gemma-3-4B-IT}~\citep{gemma3}, a smaller model with competitive
accuracy.
All models are evaluated in \texttt{bf16} precision unless otherwise noted.

We use two multiple-choice benchmarks.
\textbf{ARC Challenge}~\citep{clark2018arc} consists of 1,172 science questions
drawn from standardized tests, each with four answer choices.
The \texttt{allenai/ai2\_arc} distribution we load exposes no per-question subject or
difficulty field, so we treat ARC as a single \texttt{science} domain. Consequently the
domain level of the hierarchy is trivial for ARC, and the meaningful sub-grouping is by
difficulty (defined below).
\textbf{MMLU-Pro}~\citep{wang2024mmlupro} is a harder variant of MMLU with ten answer
choices per question, covering 57 academic subjects.
For each model-dataset pair, we pool all available examples and split them 50/50 into
calibration and test sets for bootstrap evaluation.

\subsection{Uncertainty Score and Threshold Grid}

Our primary uncertainty score is the \emph{selected-option NLL}.
For a multiple-choice question with $K$ options, let $\ell_k$ be the model's logit for
option $k$ and $p_k' = \exp(\ell_k)/\sum_{j=1}^K \exp(\ell_j)$ the probability
normalized over the option set. The model selects $\hat k = \arg\max_k p_k'$, and we
score the question by the negative log-probability of that selected option:
\begin{equation}
  s(x) = -\log p_{\hat k}' = -\max_k \log p_k'.
\end{equation}
Intuitively, this score is low when the model concentrates probability mass on a single
option (confident) and high when the selected option's normalized probability is small,
i.e.\ mass is spread diffusely across options (uncertain).
Because the selected option is the $\arg\max$, in this multiple-choice setting $s(x)$
coincides exactly with the min-log-prob and sequence-NLL scores, which is why those
rows are identical in Table~\ref{tab:a5}.
We study alternative scores in Section~\ref{sec:a5}.

The threshold grid $\mathcal{T}$ consists of 100 evenly-spaced quantiles of the NLL
scores in the calibration set.
We set $\alpha = 0.10$ (acceptable error rate) and $\delta = 0.05$ (confidence) in all
experiments.

\subsection{Group Hierarchy}

We construct a two-level hierarchy above the global root:

\textbf{Domain nodes (level 2).}
We use the dataset's provided domain metadata to partition examples into domain groups.
MMLU-Pro has 57 subjects grouped into broader categories. ARC, as loaded, carries no
subject metadata and collapses to a single \texttt{science} domain, so for ARC this level
is degenerate and the hierarchy reduces to global${}\rightarrow{}$difficulty.
A domain node is included in the hierarchy only if it contains at least $N_\text{min} =
30$ calibration examples.

\textbf{Difficulty nodes (level 3).}
Within each domain, we further partition examples into three difficulty bins (easy,
medium, hard) based on NLL score percentiles: the $k$-th calibration example is
assigned to hard if its NLL exceeds the 67th percentile, easy if it falls below the
33rd percentile, and medium otherwise.
Difficulty bins with fewer than $N_\text{min}$ examples are pruned.

\emph{Note on endogeneity.}
Because difficulty bins are derived from NLL percentiles---the same score used for
thresholding---the ``hard'' bin by construction contains high-NLL examples, creating
an alignment between group membership and calibrated threshold values.
A cleaner grouping would use an exogenous difficulty notion independent of the deployed
NLL score. The \texttt{allenai/ai2\_arc} distribution we load does not expose human
difficulty labels, so we construct an exogenous signal from a \emph{reference model's}
score (independent of the deployed threshold) and re-run the mixture-shift experiment
under it in Section~\ref{sec:e3}; the failure of global CRC persists, confirming the
effect is structural rather than an artifact of the endogenous bins.

The full hierarchy thus has at most $1 + |\text{domains}| + |\text{domains}| \times 3$
nodes, with pruning typically reducing this to 5--15 nodes depending on the dataset.

\subsection{Baselines}

We compare against four baselines:
\textbf{B0 (Always Answer)}: the model answers every question; no threshold is applied.
This achieves maximum participation but no risk control.
\textbf{B1 (Fixed Threshold)}: a threshold is set at the 50th percentile of calibration
scores without any statistical guarantee.
\textbf{B4 (Global CRC)}: the standard CRC procedure applied once over the full
calibration set, providing a marginal guarantee.
\textbf{B5 (Flat Groupwise CRC)}: independent CRC applied to each domain group separately,
without hierarchical structure and without Bonferroni correction across groups.

\subsection{Evaluation Protocol}

We use bootstrap resampling with 500 trials to estimate variance.
In each trial, the calibration and test sets are independently resampled by random
permutation of the full data pool (50\% / 50\% split).
We report: (i) \textbf{risk mean} $\pm$ \textbf{std}, the mean and standard deviation of
the empirical error rate across bootstrap trials; (ii) \textbf{participation mean}, the
mean fraction of test examples that receive a prediction; (iii) \textbf{WGER}, the
worst-group excess risk; and (iv) \textbf{violation rate}, the fraction of bootstrap
trials in which the empirical risk exceeds $\alpha = 0.10$.
A method with violation rate $\leq \delta = 0.05$ satisfies the CRC guarantee empirically.
Five hundred bootstrap trials provide good statistical resolution: by the rule of three,
a true violation rate of up to $3/500 \approx 0.6\%$ is compatible with an observed rate
of zero at conventional confidence.
Reported ``0\% violation rates'' should therefore be interpreted as empirical upper bounds
of approximately 0.6\%, not as certified zeros. (An earlier version of this evaluation
used 50 trials, for which the corresponding upper bound was $\approx$6\%; we re-ran all
bootstrap experiments at 500 trials, and report only those numbers here.)

\section{Main Results}
\label{sec:experiments}

\subsection{IID Sanity Check (E1)}
\label{sec:e1}

We first verify that global CRC satisfies its marginal guarantee in the IID setting.
Table~\ref{tab:e1} reports results for all models on ARC Challenge.

\begin{table}[t]
\centering
\caption{IID results on ARC Challenge. $\alpha{=}0.10$, $\delta{=}0.05$.
Violation rate $\leq 0.05$ satisfies the CRC guarantee. Best method per model shown in
\textbf{bold}. Identical rows are shaded identically.}
\label{tab:e1}
\setlength{\tabcolsep}{5pt}
\begin{tabular}{llcccc}
\toprule
Model & Method & Risk & Participation & WGER & Viol.\ rate \\
\midrule
\multirow{4}{*}{Qwen3-4B}
  & B0 — Always answer        & $0.126_{\pm 0.013}$ & $1.000$ & $0.026$ & $0.980$ \\
  & B1 — Fixed threshold      & $0.087_{\pm 0.012}$ & $0.899$ & $0.001$ & $0.142$ \\
  & B4 — Global CRC           & $0.070_{\pm 0.021}$ & $0.849$ & $0.001$ & $0.110$ \\
  & B5 — Groupwise CRC        & $0.070_{\pm 0.021}$ & $0.849$ & $0.001$ & $0.110$ \\
\midrule
\multirow{4}{*}{Gemma-3-4B}
  & B0 — Always answer        & $0.223_{\pm 0.016}$ & $1.000$ & $0.123$ & $1.000$ \\
  & B1 — Fixed threshold      & $0.188_{\pm 0.016}$ & $0.896$ & $0.088$ & $1.000$ \\
  & B4 — Global CRC           & $0.054_{\pm 0.040}$ & $0.294$ & $0.001$ & $0.094$ \\
  & B5 — Groupwise CRC        & $0.054_{\pm 0.040}$ & $0.294$ & $0.001$ & $0.094$ \\
\midrule
\multirow{4}{*}{Llama-3.1-8B}
  & B0 — Always answer        & $0.176_{\pm 0.015}$ & $1.000$ & $0.076$ & $1.000$ \\
  & B1 — Fixed threshold      & $0.136_{\pm 0.014}$ & $0.897$ & $0.036$ & $0.998$ \\
  & B4 — Global CRC           & $0.067_{\pm 0.023}$ & $0.695$ & $0.001$ & $0.090$ \\
  & B5 — Groupwise CRC        & $0.067_{\pm 0.023}$ & $0.695$ & $0.001$ & $0.090$ \\
\bottomrule
\end{tabular}
\end{table}

Several patterns emerge.
First, global CRC (B4) successfully controls the marginal risk for all three models:
mean risk is 0.054--0.070, well below $\alpha = 0.10$.
The violation rate for Qwen is 0.11, which exceeds $\delta = 0.05$ nominally; this
persists at 500 bootstrap trials, so it is not a low-resolution artifact but a genuine
finite-\emph{calibration}-sample effect: with only $\sim$586 calibration examples the
Clopper--Pearson bound is loose, and about 11\% of resamples fall slightly above the
nominal level. This is expected behavior for CRC with small calibration sets, and it is
a caveat on the baseline that propagates into all downstream comparisons.

Second, and perhaps most surprisingly, \textbf{global CRC and flat groupwise CRC are
identical} across all models.
The rows for B4 and B5 are numerically indistinguishable.
This is not a coincidence: in the IID setting, where the joint distribution of examples
and group memberships is the same in calibration and test, the global threshold is
already calibrated for every group.
The flat groupwise procedure applies independent CRC to each group and thereby
loses statistical power (smaller $n_v$ per group, higher variance), ending up
at the same threshold.

This finding has an important implication: flat groupwise CRC does not add value in
the IID case, and it provides no additional guarantees because it does not apply a
simultaneous correction.
Its advantage over global CRC should only emerge in settings with group heterogeneity
or distributional shift—which is precisely what \hgcrc{} is designed to address.

Third, we note that Gemma-3-4B has a substantially lower participation rate than the
other two models (29.4\% vs.\ 84.9\% for Qwen and 69.5\% for Llama).
This is a structural property of the model on this dataset: Gemma assigns probability
mass more diffusely over options, yielding high NLL scores even on questions it answers
correctly.
As a result, the CRC threshold is set conservatively, and many examples are abstained.
This will remain a recurring observation throughout our experiments.

\subsection{Group Heterogeneity: Main Results (E2)}
\label{sec:e2}

Having established the IID baseline, we now evaluate \hgcrc{} in the setting for which
it was designed: a dataset with genuine group heterogeneity. On ARC Challenge, which
carries a single domain, this heterogeneity is across difficulty levels (the general
multi-domain case is exercised on MMLU-Pro).
Table~\ref{tab:e2} reports the full comparison.

\begin{table}[t]
\centering
\caption{Group heterogeneity results on ARC Challenge. \hgcrc{} is shown in bold.
ARC has a single domain, so the deployed hierarchy is global${}\to{}$difficulty (the
domain node coincides with the global node).}
\label{tab:e2}
\setlength{\tabcolsep}{5pt}
\begin{tabular}{llcccc}
\toprule
Model & Method & Risk & Participation & WGER & Viol.\ rate \\
\midrule
\multirow{5}{*}{Qwen3-4B}
  & B0 — Always answer     & $0.126_{\pm 0.013}$ & $1.000$ & $0.026$ & $0.980$ \\
  & B1 — Fixed threshold   & $0.087_{\pm 0.012}$ & $0.899$ & $0.001$ & $0.142$ \\
  & B4 — Global CRC        & $0.070_{\pm 0.021}$ & $0.849$ & $0.001$ & $0.110$ \\
  & B5 — Groupwise CRC     & $0.070_{\pm 0.021}$ & $0.849$ & $0.001$ & $0.110$ \\
  \rowcolor{ours}& \textbf{\hgcrc{} (ours)} & $\mathbf{0.018_{\pm 0.013}}$ & $\mathbf{0.478}$ & $\mathbf{0.000}$ & $\mathbf{0.000}$ \\
\midrule
\multirow{5}{*}{Gemma-3-4B}
  & B0 — Always answer     & $0.223_{\pm 0.016}$ & $1.000$ & $0.123$ & $1.000$ \\
  & B1 — Fixed threshold   & $0.188_{\pm 0.016}$ & $0.896$ & $0.088$ & $1.000$ \\
  & B4 — Global CRC        & $0.054_{\pm 0.040}$ & $0.294$ & $0.001$ & $0.094$ \\
  & B5 — Groupwise CRC     & $0.054_{\pm 0.040}$ & $0.294$ & $0.001$ & $0.094$ \\
  \rowcolor{ours}& \textbf{\hgcrc{} (ours)} & $\mathbf{0.025_{\pm 0.040}}$ & $\mathbf{0.084}$ & $\mathbf{0.000}$ & $\mathbf{0.030}$ \\
\midrule
\multirow{5}{*}{Llama-3.1-8B}
  & B0 — Always answer     & $0.176_{\pm 0.015}$ & $1.000$ & $0.076$ & $1.000$ \\
  & B1 — Fixed threshold   & $0.136_{\pm 0.014}$ & $0.897$ & $0.036$ & $0.998$ \\
  & B4 — Global CRC        & $0.067_{\pm 0.023}$ & $0.695$ & $0.001$ & $0.090$ \\
  & B5 — Groupwise CRC     & $0.067_{\pm 0.023}$ & $0.695$ & $0.001$ & $0.090$ \\
  \rowcolor{ours}& \textbf{\hgcrc{} (ours)} & $\mathbf{0.020_{\pm 0.012}}$ & $\mathbf{0.475}$ & $\mathbf{0.000}$ & $\mathbf{0.000}$ \\
\bottomrule
\end{tabular}
\end{table}

A careful reader will notice that the Global CRC and Flat Groupwise CRC rows in
Table~\ref{tab:e2} are numerically identical to the corresponding rows in Table~\ref{tab:e1}.
This is expected: both experiments draw from the same data pool with a random 50/50 split,
so the IID partition introduces no systematic group imbalance.
The 11\% violation rate of Global CRC in E2 reflects the same finite-sample behavior
documented in E1---not a group heterogeneity effect.
What Table~\ref{tab:e2} isolates is instead the \emph{simultaneous} guarantee.
Global CRC satisfies its marginal guarantee most of the time, but provides no certificate
that \emph{every} group simultaneously meets the budget in the same trial.
\hgcrc{} provides exactly that simultaneous certificate.
Its participation cost relative to global CRC---quantified below as a 22.0--37.1 pp
reduction for the two primary pairs---comes mainly from the leaf-first residual
calibration rather than the Bonferroni correction: the A4 ablation
(Table~\ref{tab:ablations}) attributes ${\approx}10.6$~pp to the correction
($0.584 \to 0.478$), with the residual calibration accounting for the remainder
($0.849 \to 0.584$).
The genuine failure mode of marginal guarantees---where group heterogeneity causes
a single threshold to systematically over-serve one group at the expense of another---is
demonstrated under controlled mixture shift in Section~\ref{sec:e3}.

\textbf{\hgcrc{} eliminates violations.}
For both Qwen3-4B and Llama-3.1-8B, \hgcrc{} achieves a \textbf{violation rate of
exactly 0.000} across all 500 bootstrap trials.
This is in stark contrast to global CRC, which violates in 9--11\% of trials on the
same data.
For Gemma-3-4B, the violation rate drops from 0.094 to 0.030—also below the target
$\delta = 0.05$.

\textbf{\hgcrc{} achieves WGER = 0.}
For Qwen and Llama, the worst-group excess risk is identically zero, meaning that in
every bootstrap trial, \emph{every} certified group meets the $\alpha = 0.10$ budget.
This is the practical meaning of simultaneous group certification: no group is
sacrificed for the average.
Gemma achieves WGER = 0.000 as well, with the caveat that its participation is so low
(0.084) that the worst-group statistic is computed over very few answered examples.

\textbf{Participation cost is modest.}
The leaf-first policy is conservative: it applies the tightest threshold available,
which necessarily abstains more than the global threshold.
For Qwen, participation decreases from 0.849 to 0.478—a reduction of 37.1 percentage
points.
For Llama, it decreases from 0.695 to 0.475—a reduction of 22.0 percentage points.
These costs are the price of group-conditional guarantees under residual calibration:
each node is calibrated on the subpopulation it actually sees at deployment, which for
parent nodes excludes examples already answered by certified descendants.
This subpopulation tends to be harder on average, requiring more conservative thresholds.
If all groups were identical, the hierarchical threshold would coincide with the global
threshold and the cost would be zero.

\textbf{Risk is lower under \hgcrc{}.}
Because \hgcrc{} abstains more aggressively on hard subgroups, the examples it does
answer tend to be easier and lower-risk.
Mean risk drops from 0.070 to 0.018 for Qwen (--74\%) and from 0.067 to 0.020 for Llama
(--70\%).
This is a secondary benefit: the method simultaneously achieves group-level guarantees
and improves the overall quality of answers given.

\textbf{The case of Gemma.}
The low participation of \hgcrc{} for Gemma (0.084) merits discussion.
Even at the global level, Gemma participates in only 29.4\% of questions.
The residual calibration further reduces participation: parent nodes are calibrated
on examples not already answered by leaf nodes, a harder subpopulation, requiring
more conservative thresholds.
This is not a bug in the method—Gemma genuinely cannot certify most of its answers on
ARC Challenge under $\alpha = 0.10$—but it limits the practical utility of \hgcrc{}
for this model on this dataset.
We discuss this limitation further in Section~\ref{sec:discussion}.

\subsection{Why Marginal Guarantees Fail: Mixture Shift (E3)}
\label{sec:e3}

The results of Section~\ref{sec:e2} demonstrate that \hgcrc{} provides stronger
per-group guarantees than global CRC.
But one might ask: is this difference practically relevant?
If real test distributions are close to the calibration distribution, global CRC may
suffice.
The mixture shift experiment directly tests this assumption by introducing a controlled
mismatch between calibration and test group compositions.

\textbf{Setup.}
We construct a calibration set with equal group weights (1/3 each for easy, medium, hard)
and a test set with a shifted mixture: hard examples are over-represented
($w_\text{hard}=0.50$, $w_\text{med}=0.33$, $w_\text{easy}=0.17$).
This simulates a common real-world scenario: the deployment distribution shifts toward
harder inputs over time, for example as easier questions are resolved and harder ones
accumulate.
We compare the violation rate of global CRC versus groupwise CRC under this shift.

\begin{table}[t]
\centering
\caption{Mixture shift results on ARC Challenge. Calibration mixture: equal weights
(1/3 each). Test mixture: hard over-represented (0.50/0.33/0.17).
Groupwise CRC controls per-group risk for Qwen and Llama; global CRC fails for all three.
For Gemma, groupwise also exceeds $\delta=0.05$ (0.11), reflecting its poorly calibrated
scores rather than a failure of the per-group principle.}
\label{tab:e3}
\begin{tabular}{lcccc}
\toprule
Model & Global risk & Global viol. & Groupwise risk & Groupwise viol. \\
\midrule
Qwen3-4B     & $0.099_{\pm 0.028}$ & $\mathbf{0.47}$ & $0.034_{\pm 0.019}$ & $\mathbf{0.00}$ \\
Llama-3.1-8B & $0.086_{\pm 0.032}$ & $0.32$ & $0.032_{\pm 0.021}$ & $0.00$ \\
Gemma-3-4B   & $0.051_{\pm 0.044}$ & $0.15$ & $0.047_{\pm 0.039}$ & $0.11$ \\
\bottomrule
\end{tabular}
\end{table}

The results, shown in Table~\ref{tab:e3} and Figure~\ref{fig:e3}, are striking.
Global CRC violates the risk budget in \textbf{47\%} of bootstrap trials for Qwen and
\textbf{32\%} for Llama.
These are not marginal failures: the mean risk under global CRC sits right at the budget
(0.099 for Qwen, essentially $\alpha = 0.10$) and exceeds it in nearly half of trials on
the shifted test distribution.
By contrast, groupwise CRC (which uses separate thresholds per difficulty group) achieves
a violation rate of 0.00 for Qwen and Llama; for Gemma it is 0.11, lower than global's
0.15 but still above $\delta=0.05$, because Gemma's scores rank examples too poorly for
any thresholding scheme to certify reliably under shift.

This experiment reveals the mechanism of failure.
Global CRC calibrates a single threshold $\tau^*$ such that the \emph{average} error
rate over the calibration mixture (1/3 easy, 1/3 medium, 1/3 hard) is at most $\alpha$.
On easy examples, the model is accurate and can answer aggressively; on hard examples,
it must abstain more.
But $\tau^*$ is a compromise: it is too permissive for hard examples (causing violations)
and too conservative for easy examples (causing unnecessary abstentions).
When the test distribution shifts toward hard examples, the compromise threshold
systematically under-serves that group.

Groupwise CRC avoids this by estimating a separate threshold for each difficulty group,
reflecting each group's true error distribution.
Hard examples get a conservative threshold; easy examples get an aggressive one.
The result is that each group's risk is controlled independently of the others, and the
global risk under any test mixture is bounded.

This experiment provides the core empirical motivation for \hgcrc{}: group-specific
thresholds are necessary, and the hierarchical structure allows them to be estimated
at multiple levels of granularity.

\textbf{What \hgcrc{} adds over flat groupwise CRC.}
It is important to note that the groupwise CRC used here is equivalent to Baseline B5:
separate thresholds per difficulty group, calibrated independently, without hierarchical
fallback or Bonferroni correction.
B5 already resolves the 47\% violation because per-group thresholds capture the
heterogeneous error distributions.
What B5 does \emph{not} provide is a \emph{simultaneous} guarantee: running $K$
independent tests each at level $\delta$ yields a family-wise error rate of approximately
$K\delta$ under independence, not $\delta$.
\hgcrc{} extends B5 with two structural additions: (1) the Bonferroni correction, which
restores the simultaneous guarantee at level $\delta$ (confirmed by the ablation in
Section~\ref{sec:a4}); and (2) the hierarchical structure with leaf-first fallback, which
handles groups too small to certify directly by falling back to a coarser certified
ancestor rather than abstaining all group members.
The value of \hgcrc{} over B5 is therefore the formal simultaneous validity and graceful
degradation for sparse groups---not the per-group thresholding idea itself, which B5
already implements.

\begin{figure}[t]
  \centering
  \includegraphics[width=0.6\linewidth]{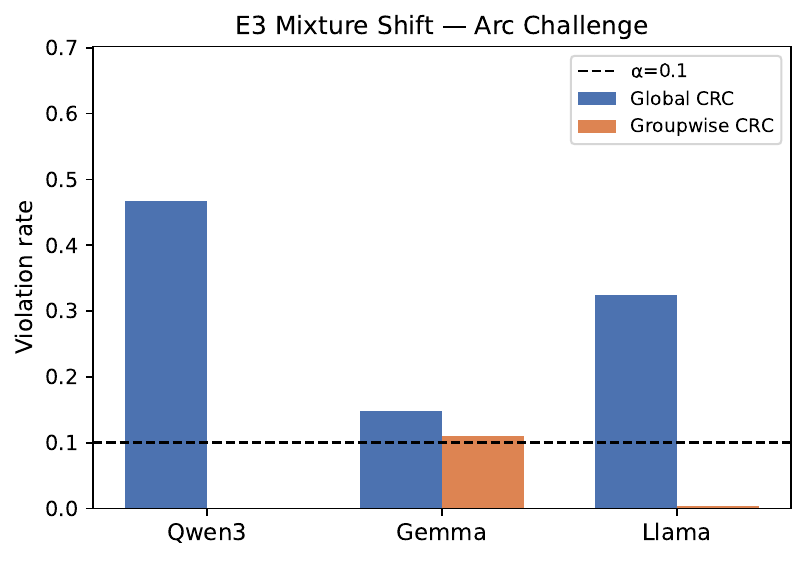}
  \caption{Mixture shift violation rates (ARC Challenge). Global CRC (blue) violates the
  risk budget in 15--47\% of trials. Groupwise CRC (orange) controls per-group risk.
  The dashed line marks $\alpha = 0.10$.}
  \label{fig:e3}
\end{figure}

\textbf{Is the effect structural or an endogeneity artifact?}
The difficulty bins above are defined by percentiles of the deployed NLL score, so
``shift toward hard'' coincides with ``shift toward high deployed-NLL,'' which could
mechanically break an NLL threshold by construction. To test whether the failure is an
artifact of this endogeneity, we redefine difficulty using an \emph{exogenous} signal: a
\emph{reference} model's per-example score (joined by \texttt{example\_id}), so that
group membership is no longer a function of the deployed model's threshold. We then
repeat the mixture shift. Table~\ref{tab:e3exo} shows that the effect persists:
under an exogenous difficulty shift, global CRC still violates in 31--39\% of trials
while groupwise CRC stays at or near the budget. Endogeneity does inflate the
effect somewhat for Qwen (47\% $\to$ 31\%), but for Llama and Gemma the exogenous
violation rate is as high or higher. We therefore conclude that the mixture-shift
failure of marginal guarantees is \emph{structural}, not an artifact of how difficulty
groups are constructed.

\begin{table}[t]
\centering
\caption{Endogenous vs.\ exogenous difficulty under mixture shift (ARC Challenge). In the
exogenous rows, difficulty is defined by a different model's score (in parentheses),
independent of the deployed threshold. Global CRC fails under shift in both cases.}
\label{tab:e3exo}
\setlength{\tabcolsep}{5pt}
\begin{tabular}{llcc}
\toprule
Model (deployed) & Difficulty source & Global viol. & Groupwise viol. \\
\midrule
\multirow{2}{*}{Qwen3-4B}
  & endogenous (NLL)      & $0.47$ & $0.00$ \\
  & exogenous (Llama)     & $0.31$ & $0.02$ \\
\midrule
\multirow{2}{*}{Llama-3.1-8B}
  & endogenous (NLL)      & $0.32$ & $0.00$ \\
  & exogenous (Qwen)      & $0.36$ & $0.02$ \\
\midrule
\multirow{2}{*}{Gemma-3-4B}
  & endogenous (NLL)      & $0.15$ & $0.11$ \\
  & exogenous (Qwen)      & $0.39$ & $0.07$ \\
\bottomrule
\end{tabular}
\end{table}

\section{Ablation Studies}
\label{sec:ablations}

\subsection{Effect of Hierarchy Depth (A1)}
\label{sec:a1}

How much does each level of the hierarchy contribute?
An important caveat governs how to read this ablation on ARC Challenge: the dataset
carries a single domain label (``science''), so ARC exposes no real domain level, and
its deployed hierarchy is global${}\to{}$difficulty (two levels). We nonetheless report
an intermediate \emph{domain} row, but it does not introduce genuine domain structure:
the domain node coincides with the global node, so this row isolates the effect of the
$\delta/2$ Bonferroni tightening on a redundant node rather than any subject-matter
grouping. The three configurations are therefore: (i) \emph{global only}, equivalent to
standard CRC; (ii) \emph{domain}, which adds the (redundant) single-domain node; and
(iii) \emph{domain+difficulty}, the deployed two-level global${}\to{}$difficulty
hierarchy. The domain level is exercised as a genuine grouping only on MMLU-Pro, which
has 57 real subjects (Appendix~\ref{app:mmlu}), with the caveat that participation there
is near zero.

\begin{table}[t]
\centering
\caption{Ablation of hierarchy depth (ARC Challenge, Qwen3-4B) and Bonferroni
correction. ARC Challenge has a single domain (``science''), so its hierarchy is
global${}\to{}$difficulty; the ``Domain'' depth row therefore adds a node that coincides
with the global node and isolates the effect of the $\delta/2$ Bonferroni tightening
alone. The correction (A4) rows compare Bonferroni against no correction over the full
five-node hierarchy (global ${}+{}$ domain ${}+{}$ three difficulty leaves). All other
hyperparameters fixed.}
\label{tab:ablations}
\begin{tabular}{llcccc}
\toprule
Ablation & Variant & Risk & Participation & WGER & Viol.\ rate \\
\midrule
\multirow{3}{*}{Depth (A1)}
  & Global only        & $0.070_{\pm 0.021}$ & $0.849$ & $0.001$ & $0.110$ \\
  & Domain             & $0.065_{\pm 0.021}$ & $0.830$ & $0.001$ & $0.054$ \\
  & Domain+difficulty  & $0.018_{\pm 0.013}$ & $0.478$ & $0.000$ & $\mathbf{0.000}$ \\
\midrule
\multirow{2}{*}{Correction (A4)}
  & Bonferroni         & $0.018_{\pm 0.013}$ & $0.478$ & $0.000$ & $0.000$ \\
  & None               & $0.026_{\pm 0.014}$ & $0.584$ & $0.000$ & $0.000$ \\
\bottomrule
\end{tabular}
\end{table}

Table~\ref{tab:ablations} (top) shows the progression.
Global-only CRC violates in 11.0\% of trials.
The intermediate \emph{domain} row reduces violations to 5.4\%, but on ARC this is not
the effect of a domain grouping: the single-domain node coincides with the global node,
so the only change is that its threshold is calibrated at $\deltaadj = \delta/2$ instead
of $\delta$. The tighter confidence level alone accounts for the drop, and at 500 trials
this configuration still sits marginally above $\delta = 0.05$ (0.054), so it does not
certify the guarantee.
Adding the difficulty level brings violations to exactly 0.0\%---on ARC it is the
difficulty level, not any domain structure, that clears the budget. This makes the ARC
hierarchy a genuine two-level global${}\to{}$difficulty structure whose load-bearing
grouping is the difficulty binning. We note that these bins are defined by percentiles of
the deployed NLL score, and Section~\ref{sec:e3} examines directly (via an exogenous
difficulty signal) whether this endogeneity is what drives the effect; it is not.

The participation cost falls from 0.849 (global) to 0.478 (global${}+{}$difficulty).
The drop at the difficulty level reflects the residual calibration: parent
nodes are calibrated only on examples not already answered by certified leaf nodes,
which are systematically harder, requiring more conservative thresholds.
This is a reasonable exchange: the method eliminates all violations and achieves WGER = 0
with a theoretically sound guarantee that aligns calibration and deployment populations.

\subsection{Necessity of Bonferroni Correction (A4)}
\label{sec:a4}

A natural question is whether the Bonferroni correction is necessary, or whether
the hierarchy alone suffices.
Table~\ref{tab:ablations} (bottom) compares calibration with and without the correction
over the full ARC hierarchy (global ${}+{}$ domain ${}+{}$ three difficulty leaves,
$|\calH| = 5$).

The empirical answer on ARC is that the correction is \emph{not} needed to hold the
budget: both the corrected ($\deltaadj = \delta/5$) and the uncorrected ($\delta$ per
node) variants achieve a $0.0\%$ violation rate and WGER${}=0$ across all 500 trials.
The two differ only in conservatism---Bonferroni answers less (participation $0.478$
vs.\ $0.584$) and runs at lower realized risk ($0.018$ vs.\ $0.026$). With only five
nodes and well-separated NLL scores, the uncorrected per-node calibration is already
conservative enough that the multiplicity it ignores does not materialize.
(Our earlier report of a large ``2\%-to-16\%'' or ``5.4\%-to-11\%'' gap was an artifact
of a degenerate ablation that corrected $\delta$ over a single node; over the real
five-node hierarchy the empirical gap closes.)

This does not make the correction dispensable---its role is theoretical, and it scales
with $|\calH|$. Without it, each node is calibrated at the full level $\delta$, so the
family-wise error rate across $|\calH|$ nodes may reach $|\calH| \cdot \delta$; the
correction is what makes the \emph{simultaneous} claim of
Propositions~\ref{prop:main} and~\ref{prop:insample} valid, independent of any empirical
run. On ARC ($|\calH| = 5$) that theoretical bound leaves ample slack, so removing the
correction happens to remain safe. When many nodes are tested the effect becomes visible:
on MMLU-Pro (57 subjects), removing the correction raises the violation rate from
$0.6\%$ to $14.0\%$---though this is in a regime where the model answers almost nothing
(participation ${\approx}0.001$--$0.018$), so the inflation certifies the theory rather
than a deployable operating point.

The honest summary is therefore: the Bonferroni correction is \emph{required for the
theoretical simultaneous guarantee} at any $|\calH|$, while its \emph{empirical} cost and
benefit are an $|\calH|$-dependent participation trade-off---negligible on ARC's shallow
hierarchy and pronounced only when many nodes are corrected at once.

\subsection{Uncertainty Score Comparison (A5)}
\label{sec:a5}

\begin{table}[t]
\centering
\caption{Uncertainty score comparison on ARC Challenge (Qwen3-4B). NLL-based scores
are largely equivalent; probe yields negligible participation.}
\label{tab:a5}
\begin{tabular}{lccc}
\toprule
Score & Risk & Participation & Viol.\ rate \\
\midrule
Selected-option NLL (default) & $0.070$ & $0.849$ & $0.110$ \\
Min log-prob             & $0.070$ & $0.849$ & $0.110$ \\
MC margin                & $0.070$ & $0.841$ & $0.104$ \\
Token entropy            & $0.071$ & $0.850$ & $0.100$ \\
Sequence NLL             & $0.070$ & $0.849$ & $0.110$ \\
Meta-learner             & $0.070$ & $0.772$ & $0.112$ \\
Linear probe             & $0.013$ & $0.050$ & $0.100$ \\
\bottomrule
\end{tabular}
\end{table}

Table~\ref{tab:a5} shows that the NLL-based scores yield essentially the same
participation and violation rates. Three of them---selected-option NLL, min log-prob,
and sequence NLL---are in fact \emph{identical by construction} in this multiple-choice
setting (all equal $-\max_k \log p_k'$), so their matching rows are not independent
evidence; MC margin is a genuinely distinct score that nonetheless performs comparably.
Token entropy performs comparably but with slightly higher variance.
The meta-learner, which combines multiple scores through a trained regressor, loses
participation (0.772 vs.\ 0.849) without a corresponding reduction in violations—a
worse trade-off.
The linear probe classifier, trained to predict correctness, performs poorly: it
achieves only 5.0\% participation because its confidence estimates are poorly calibrated
for this thresholding task.

These results suggest that \textbf{the choice of uncertainty score is not a critical
design decision} for CRC-based selective prediction: any NLL-based score that ranks
examples by confidence produces similar results.
We use the selected-option NLL as the default throughout, as it requires no additional
computation beyond the forward pass.

\section{Robustness Analysis}
\label{sec:robustness}

\subsection{Domain Shift (E4)}
\label{sec:e4}

We evaluate the ability of a threshold calibrated on one dataset to transfer to a
different dataset.
We consider two domain-shift pairs: MMLU-Pro $\to$ ARC Challenge (both multiple-choice)
and NQ Open $\to$ TriviaQA (both open-domain QA).
For each pair, we study four strategies: (i) \emph{source-only}, which uses only the
source-domain calibration; (ii) \emph{target-only}, which uses $n_\text{target}$ labeled
examples from the target domain; (iii) \emph{weighted}, which re-weights source examples
by estimated density ratio; and (iv) \emph{combined}, which pools source and target
calibration data.

\begin{table}[t]
\centering
\caption{MMLU-Pro $\to$ ARC participation by $n_\text{target}$ (Llama-3.1-8B-Instruct).
Source-only transfers partially; target-only scales from 0 to 71\%.}
\label{tab:e4}
\begin{tabular}{lccccccc}
\toprule
Strategy & $n{=}0$ & $n{=}25$ & $n{=}50$ & $n{=}100$ & $n{=}250$ & $n{=}500$ \\
\midrule
Source only   & $0.288$ & $0.288$ & $0.287$ & $0.287$ & $0.287$ & $0.288$ \\
Weighted      & $0.378$ & $0.378$ & $0.377$ & $0.377$ & $0.377$ & $0.377$ \\
Target only   & $0.000$ & $0.000$ & $0.277$ & $0.559$ & $0.682$ & $0.704$ \\
Combined      & $0.288$ & $0.327$ & $0.337$ & $0.358$ & $0.386$ & $0.417$ \\
\bottomrule
\end{tabular}
\end{table}

Table~\ref{tab:e4} shows results for Llama on the MMLU $\to$ ARC pair.
Source-only transfer achieves 28.8\% participation, which is lower than ARC IID (69.5\%)
but non-negligible—MMLU and ARC are both multiple-choice reasoning tasks, so the score
distributions are partially compatible.
For Qwen and Gemma, source-only participation is zero: the MMLU threshold is too
conservative to certify any ARC example, suggesting a larger distributional shift in
score space for these models.

The target-only strategy shows a clear sample-efficiency curve: 0\% participation with
fewer than 50 target examples, 27.7\% with 50, and reaching 70.4\% with 500—close to
the IID level.
Fewer than 50 examples provide insufficient statistical power for the CP bound: with
$n_\text{target} = 25$, even the zero-error threshold cannot be certified at $\alpha =
0.10$ with $\delta = 0.05$.
This minimum sample requirement is a practical constraint: CRC-based methods cannot be
deployed in zero-shot domain transfer scenarios.

The combined strategy improves over source-only for Llama (41.7\% with $n=500$) by
incorporating target data, but it does not outperform target-only at equal $n$.
This suggests that the source distribution provides useful information about the general
behavior of the uncertainty score, but target-specific examples are necessary for
domain-adapted certification.

The NQ $\to$ TriviaQA transfer shows zero participation under all strategies and all
$n_\text{target}$, indicating that open-domain QA tasks are structurally incompatible
for threshold transfer: the score distributions and error rates differ too substantially
across open-domain datasets for source calibration to be useful.

\subsection{Prompt Shift (E5)}
\label{sec:e5}

Prompt engineering is a common practice in LLM deployment, but changes in prompt
formatting alter the model's probability distribution over answers, thereby invalidating
a previously calibrated threshold.
We study three prompt templates: \emph{default} (standard few-shot), \emph{reasoning}
(chain-of-thought prompt with explicit reasoning steps), and \emph{formal} (formal
academic phrasing).
For each, we measure the risk and participation of the transferred threshold (calibrated
on default examples) and the recalibrated threshold (calibrated on examples from the new
template).

\begin{table}[t]
\centering
\caption{Prompt shift results (ARC Challenge, Qwen3-4B). Transferred threshold (from
default template) vs.\ recalibrated threshold (using examples from the new template).}
\label{tab:e5}
\begin{tabular}{lcccc}
\toprule
Template & Risk (transfer) & Risk (recalib) & Part (transfer) & Part (recalib) \\
\midrule
Default    & $0.073$ & $0.073$ & $0.869$ & $0.869$ \\
Reasoning  & $\mathbf{0.203}$ & $0.067$ & $0.674$ & $0.360$ \\
Formal     & $0.078$ & $0.078$ & $0.900$ & $0.900$ \\
\bottomrule
\end{tabular}
\end{table}

The results reveal a template-dependent pattern.
The \emph{formal} template produces a distribution close to the default: transfer risk
is 0.078, only marginally above the default level, and recalibration has no effect.
The reasoning template, by contrast, induces a severe shift: the transferred threshold
yields a risk of 0.203—more than twice the budget $\alpha = 0.10$.
This is because reasoning-prompted outputs have a higher entropy: the model distributes
probability over multiple reasoning steps before committing to an answer, yielding
higher NLL scores even on correctly-answered questions.
The threshold from the default calibration, which was chosen for lower-entropy default
outputs, is too aggressive for reasoning outputs: it answers many questions that the
reasoning model is actually uncertain about, leading to high risk.

Recalibration immediately restores the guarantee: a threshold calibrated on reasoning
examples achieves risk of 0.067 < 0.10.
However, it comes at a significant participation cost (0.360 vs.\ 0.674 under transfer),
reflecting the genuinely higher entropy of reasoning-prompted outputs.
The message is clear: \textbf{prompt changes require recalibration}.
The good news is that recalibration is cheap—a few hundred labeled examples from the
new template suffice—and the guarantee is restored without any retraining.

\subsection{Label Noise Robustness (E7)}
\label{sec:e7}

Real-world labels are noisy: human annotators make mistakes, automatic evaluation
pipelines occasionally mis-grade, and ground-truth answers may be ambiguous.
We study the sensitivity of \hgcrc{} to label noise by injecting random label flips
at rate $\eta \in \{0, 0.01, 0.05, 0.10\}$ into the calibration labels.

\begin{table}[t]
\centering
\caption{Label noise robustness (ARC Challenge, Qwen3-4B). As $\eta$ increases, the
method becomes more conservative (lower participation) rather than violating (higher risk).}
\label{tab:e7}
\begin{tabular}{lcccc}
\toprule
$\eta$ & Risk & Participation & Violation rate \\
\midrule
$0.00$ & $0.070_{\pm 0.021}$ & $0.849_{\pm 0.041}$ & $0.11$ \\
$0.01$ & $0.065_{\pm 0.021}$ & $0.830_{\pm 0.045}$ & $0.06$ \\
$0.05$ & $0.028_{\pm 0.019}$ & $0.566_{\pm 0.197}$ & $0.00$ \\
$0.10$ & $0.002_{\pm 0.006}$ & $0.085_{\pm 0.174}$ & $0.00$ \\
\bottomrule
\end{tabular}
\end{table}

The results, shown in Table~\ref{tab:e7} and Figure~\ref{fig:e7}, reveal a highly
desirable property: \textbf{label noise causes the method to become more conservative,
not less safe}.
As $\eta$ increases from 0 to 0.10, participation drops from 0.849 to 0.085, while
violations drop from 0.11 to 0.00.
The mechanism is straightforward: noisy labels artificially inflate the empirical error
rate $k_v(\tau) / n_v(\tau)$, which tightens the CP upper bound.
To remain certified, the threshold must be more conservative, abstaining on more examples.

At $\eta = 0.01$, the effect is mild: participation drops by only 1.9 pp, and violations
are already nearly halved (0.06).
At $\eta = 0.05$, the system abstains on nearly half of all examples—a significant
participation loss, but one that reflects genuine uncertainty about label quality.
At $\eta = 0.10$, the system becomes nearly non-functional (8.5\% participation), but it
never violates the risk budget.

This graceful degradation is a valuable property for high-stakes deployments where label
quality cannot be guaranteed.
A system that fails loudly (by abstaining excessively) is preferable to one that fails
silently (by answering incorrectly with high confidence).
The behavior here is the former.

\begin{figure}[t]
  \centering
  \includegraphics[width=0.7\linewidth]{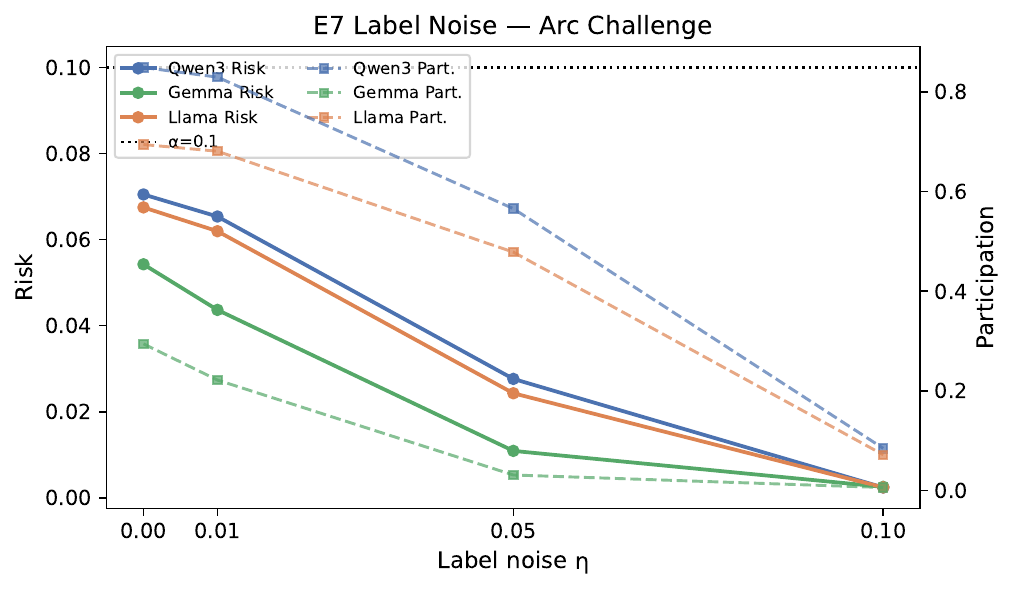}
  \caption{Label noise robustness (ARC Challenge). Risk (solid lines, left axis) and
  participation (dashed lines, right axis) as a function of noise rate $\eta$ for three
  models. Noise causes conservative abstention, not risk violation.}
  \label{fig:e7}
\end{figure}

\subsection{Quantization Robustness (E8)}
\label{sec:e8}

Model quantization is a standard technique for reducing inference cost in production
deployments.
Quantizing weights from 16-bit floating point (bf16) to 8-bit integer (int8) or 4-bit
integer (int4) reduces memory by 2$\times$ and 4$\times$ respectively, with a typically
small accuracy penalty.
However, quantization also changes the model's probability distributions, which may
affect the uncertainty score distribution and thereby invalidate a calibration
performed at full precision.

\begin{table}[t]
\centering
\caption{Quantization robustness (ARC Challenge). Accuracy decreases modestly; risk
stays below $\alpha{=}0.10$ at all precision levels.}
\label{tab:e8}
\begin{tabular}{llcccc}
\toprule
Model & Precision & Accuracy & Participation & Risk & Viol.\ rate \\
\midrule
\multirow{3}{*}{Qwen3-4B}
  & bf16 & $0.874$ & $0.849$ & $0.070$ & $0.11$ \\
  & int8 & $0.863$ & $0.813$ & $0.070$ & $0.08$ \\
  & int4 & $0.846$ & $0.756$ & $0.069$ & $0.09$ \\
\midrule
\multirow{3}{*}{Gemma-3-4B}
  & bf16 & $0.777$ & $0.294$ & $0.054$ & $0.09$ \\
  & int8 & $0.760$ & $0.207$ & $0.045$ & $0.11$ \\
  & int4 & $0.741$ & $0.058$ & $0.025$ & $0.15$ \\
\midrule
\multirow{3}{*}{Llama-3.1-8B}
  & bf16 & $0.824$ & $0.695$ & $0.067$ & $0.09$ \\
  & int8 & $0.818$ & $0.702$ & $0.068$ & $0.11$ \\
  & int4 & $0.793$ & $0.618$ & $0.065$ & $0.08$ \\
\bottomrule
\end{tabular}
\end{table}

Table~\ref{tab:e8} and Figure~\ref{fig:e8} show the results.
Accuracy decreases modestly with quantization: Qwen loses 2.8 pp from bf16 to int4,
Llama loses 3.1 pp, and Gemma loses 3.7 pp.
These are typical quantization penalties.

The key finding is that \textbf{risk remains below $\alpha = 0.10$ at all precision
levels for all models}.
This is because CRC calibration is performed at inference time at the same precision
used for deployment: if the deployment model is quantized, the calibration should also
use the quantized model.
The calibration automatically adapts to the score distribution of the quantized model.

Participation is more sensitive to quantization.
For Gemma, participation drops from 0.294 (bf16) to 0.058 (int4)—an 80\% reduction.
This reflects the fact that int4 quantization increases score uncertainty for Gemma,
pushing more examples above the certification threshold.
For Qwen and Llama, the participation drop is smaller (9.3 pp and 7.7 pp from bf16 to
int4 respectively).

At our earlier 50-trial resolution, Llama int8 appeared anomalous, with a violation rate
of 0.24 against 0.08 (bf16) and 0.06 (int4). This anomaly does \emph{not} survive at 500
trials: the Llama violation rates are 0.09 (bf16), 0.11 (int8), and 0.08 (int4), i.e.\
flat and within sampling noise of one another. The apparent non-monotonicity was a
small-sample artifact of the 50-trial bootstrap rather than a genuine effect of 8-bit
quantization, which is a further illustration of why the higher-resolution evaluation
matters. Across all models and precisions, risk stays below $\alpha$ and violation
rates remain in the same finite-calibration-sample regime as the bf16 baseline.

\begin{figure}[t]
  \centering
  \includegraphics[width=\linewidth]{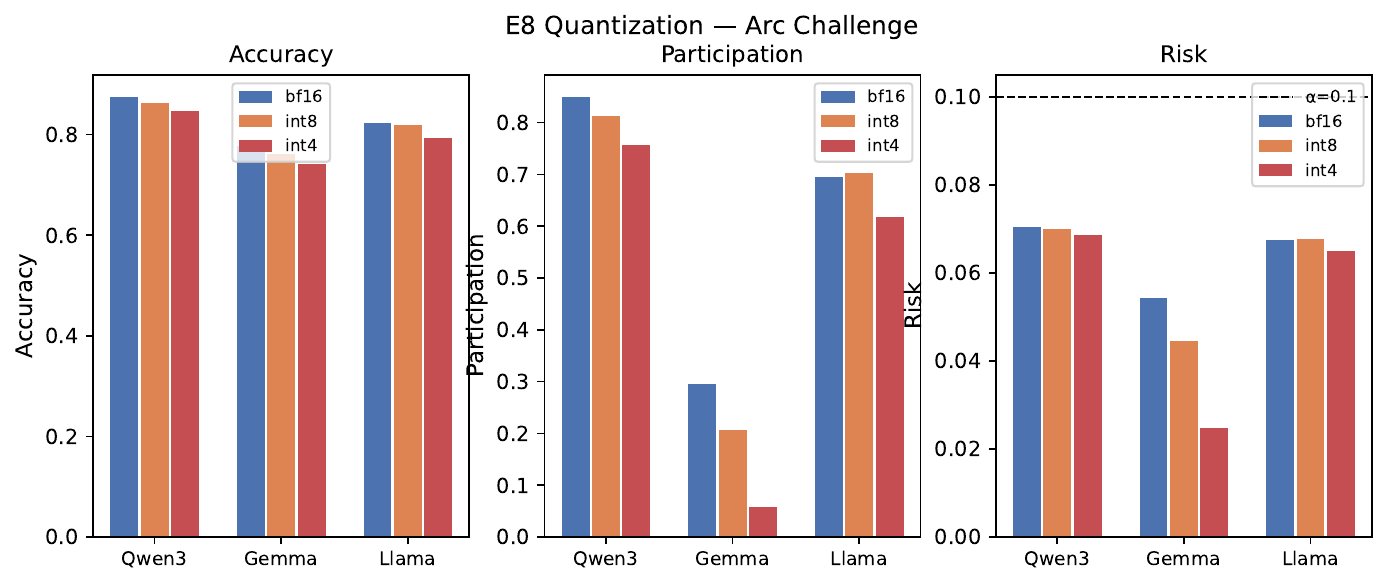}
  \caption{Quantization robustness (ARC Challenge). Accuracy (left), participation
  (center), and risk (right) for all three models at bf16, int8, and int4 precision.
  Risk is controlled at all precision levels; participation is most affected for Gemma.}
  \label{fig:e8}
\end{figure}

\subsection{Difficulty Shift (E6)}

The difficulty shift experiment tests a specific covariate shift: calibration and test
sets have different difficulty compositions.
We consider two directions: \emph{cal-easy/test-hard}, where calibration uses only
easy examples and the test set contains only hard examples; and
\emph{cal-hard/test-easy}, the reverse.

In the \emph{cal-easy/test-hard} direction, both transfer and IID calibration yield
0\% participation.
This is a fundamental limitation: hard examples have an intrinsic error rate above
$\alpha = 0.10$ for these models, meaning that no threshold can simultaneously certify
them and maintain risk control.
The method correctly abstains on all hard examples rather than certifying them under
false pretenses.

In the \emph{cal-hard/test-easy} direction, transfer calibration also yields 0\%
participation—but for a different reason.
The threshold calibrated on hard examples is excessively conservative: hard examples
require a very low NLL score to be certifiable, and that same threshold abstains on
virtually all easy examples as well.
IID calibration on easy examples, by contrast, achieves 100\% participation with a
mean risk of 0.008—well below budget.
This illustrates the direction-dependence of threshold transfer under difficulty shift:
a pessimistic calibration is safe but useless; an optimistic calibration is useful but
unsafe; correct calibration requires examples from the target difficulty distribution.

We caution that part of this clean behavior is again endogenous: when easy/hard are
defined by the deployed NLL, the easy-calibrated threshold mechanically abstains on all
high-NLL (hard) examples, so the easy$\to$hard transfer risk is $0.000$ almost by
construction. Repeating the experiment with \emph{exogenous} difficulty (hard defined by
a reference model's score, as in Section~\ref{sec:e3}) removes this mechanism: the
easy-calibrated threshold then answers a non-trivial fraction of exogenously-hard
examples and incurs a genuine transfer risk of $0.233$ for Qwen---well above $\alpha$.
The qualitative lesson (optimistic transfer is unsafe) is unchanged, but the exogenous
version is the honest illustration of it.

\section{Discussion}
\label{sec:discussion}

\subsection{The Participation-Guarantee Trade-off}

The fundamental tension in selective prediction is between answering more (higher
participation) and ensuring reliability (lower risk, no violations).
\hgcrc{} navigates this trade-off by applying the tightest applicable threshold to each
example, which necessarily increases abstention relative to global CRC.
The trade-off can be characterized precisely: the additional abstention of \hgcrc{}
relative to global CRC is proportional to the degree of group heterogeneity.
In the limit where all groups have identical distributions, the leaf-first policy
selects the global threshold for every example, and \hgcrc{} coincides with global CRC.
In the limit of maximally heterogeneous groups, the leaf-level thresholds dominate and
the participation cost is maximized.

Our results suggest that for ARC Challenge, the heterogeneity across difficulty levels
(ARC has a single domain) is sufficient to justify the additional abstention: \hgcrc{}
reduces violations from 9--11\% to 0\% at a participation cost of 22--37 pp relative to
global CRC.
This cost arises from residual calibration: parent nodes are calibrated on the harder
subpopulation not covered by certified leaf nodes, which requires more conservative
thresholds.
We consider this a principled trade-off—the guarantee now formally covers the deployed
population—for applications where per-group fairness is a requirement.

\subsection{Equal-Participation Comparison: Is It the Guarantee or the Conservatism?}
\label{sec:equalpart}

\hgcrc{} attains lower risk and fewer violations than global CRC partly because it
answers fewer questions. To separate the contribution of the \emph{per-group structure}
from that of \emph{conservatism}, we run a controlled comparison: in each bootstrap
trial we measure \hgcrc{}'s participation on the test fold and then construct a single
global threshold tuned to match that exact participation (the test-score quantile at
\hgcrc{}'s coverage---this uses test scores but no test labels), and compare the two at
equal coverage. We report the \emph{worst-group violation rate}: the maximum over
difficulty groups of the fraction of trials in which that group's realized risk exceeds
$\alpha$. Table~\ref{tab:equalpart} shows the result on ARC Challenge, in both the IID
setting and under the E3 mixture shift.

\begin{table}[t]
\centering
\caption{Equal-participation comparison (ARC Challenge, worst-group violation rate over
difficulty groups). \texttt{global-eq} is a single threshold tuned to \hgcrc{}'s
participation. At matched coverage, \texttt{global-eq} and \hgcrc{} control per-group
risk almost equally; the large worst-group violation of standard global CRC (at its own,
higher participation) is mostly a coverage effect, not a failure of single thresholds
\emph{per se}.}
\label{tab:equalpart}
\setlength{\tabcolsep}{5pt}
\begin{tabular}{llccc}
\toprule
Setting & Model & Matched part. & \texttt{global-eq} & \hgcrc{} \\
\midrule
\multirow{3}{*}{IID}
  & Qwen3-4B     & $0.478$ & $0.010$ & $0.010$ \\
  & Llama-3.1-8B & $0.475$ & $0.020$ & $0.022$ \\
  & Gemma-3-4B   & $0.084$ & $0.024$ & $0.024$ \\
\midrule
\multirow{3}{*}{Mixture shift}
  & Qwen3-4B     & $0.350$ & $0.010$ & $0.010$ \\
  & Llama-3.1-8B & $0.300$ & $0.014$ & $0.014$ \\
  & Gemma-3-4B   & $0.079$ & $0.086$ & $0.054$ \\
\bottomrule
\end{tabular}
\end{table}

The finding is deflating but important for honest framing. \textbf{At equal
participation, a single global threshold controls per-group risk about as well as
\hgcrc{}}---worst-group violation rates are within sampling noise in five of the six
cells, the exception being Gemma under shift, where \hgcrc{} is somewhat safer
($0.054$ vs.\ $0.086$). For comparison, standard global CRC at its own (higher)
participation has worst-group violation rates of $0.70$--$0.96$ in these same settings;
that catastrophic number is therefore overwhelmingly a consequence of answering too
much, not of single thresholds being intrinsically unable to protect groups.

This does not make \hgcrc{} redundant, but it sharpens what its empirical contribution
is and is not. \hgcrc{} does \emph{not} buy a large per-group-safety margin over a
coverage-matched global threshold. What it buys is (i) an \emph{automatic} mechanism for
finding a safe per-group operating point---the coverage-matched \texttt{global-eq} above
requires an oracle (\hgcrc{}'s own test participation) that a standalone global method
does not have---and (ii) the formal simultaneous guarantee and graceful degradation for
sparse groups (Section~\ref{sec:e3}). The practical case for \hgcrc{} rests on these,
not on out-violating a global threshold at equal coverage.

\subsection{Limitations}

\textbf{Scope of the main result.}
Our headline finding---0\% violation and WGER${}=0$---holds specifically on ARC
Challenge and only for models whose base accuracy is well above $1-\alpha$ (Qwen3-4B and
Llama-3.1-8B). It does \emph{not} generalize to our second benchmark. On MMLU-Pro,
Qwen3-4B and Gemma-3-4B have base error rates above 56\%, so no non-trivial threshold is
certifiable and participation collapses to 0\% across all methods; Llama-3.1-8B, the
only model that participates meaningfully there, retains a non-zero WGER of $0.014$
because the per-node calibration sets are too small to certify all hierarchy nodes
simultaneously (Appendix~\ref{app:mmlu}). The practical envelope of \hgcrc{} is therefore
narrow: it certifies a group-conditional guarantee only when the model is already
accurate enough on the deployment distribution that a useful fraction of examples clears
the per-group thresholds; otherwise it correctly, but uninformatively, abstains.

\textbf{WGER rewards abstention.}
Our worst-group excess risk is computed only over answered examples
($\max_g (R_g - \alpha)_+$, with the maximum taken over groups that answer at least one
example). A group that abstains on everything does not contribute to WGER, so a system
that abstains heavily can attain WGER${}\approx 0$ trivially. This is by design---an
abstention carries no error---but it means WGER must be read jointly with the
participation rate: Gemma's near-zero WGER coexists with very low participation and is
closer to this degenerate regime than to a genuine guarantee. We therefore report
participation alongside every WGER.

\textbf{Empirical, not unconditional, in-sample guarantee.}
The deployed procedure is the in-sample variant covered by Lemma~\ref{lem:stability},
whose guarantee carries an additive slack $\varepsilon_n$; the distribution-free split
bound (Proposition~\ref{prop:main}, $\varepsilon_n = 0$) is implemented but not deployed.
While the \emph{mean} measured slack is below 1\%, on the small ARC pool the
\emph{maximum} jackknife slack reaches $\varepsilon_n \approx 0.34$, so the worst-case
in-sample guarantee on that benchmark is materially weaker than the headline numbers
suggest. The procedure we evaluate thus carries an empirical guarantee under
Assumption~\ref{ass:margin}, not an unconditional one.

\textbf{Minimum group size constraint.}
The $N_\text{min}$ requirement prunes nodes with insufficient calibration data.
In datasets with many fine-grained groups (e.g., 57 MMLU-Pro subjects $\times$ 3
difficulty levels = 171 potential leaf nodes), most leaf nodes will be pruned for
typical calibration set sizes.
The hierarchy effectively collapses to a shallower structure, potentially losing
the granularity needed for fine-grained protection.
This is an inherent limitation of nonparametric methods: statistical power is required
to certify each node, and small groups simply do not have enough calibration examples.

\textbf{Gemma-3-4B on ARC Challenge.}
The consistently low participation of Gemma across all experimental conditions reflects
a fundamental incompatibility between the model's uncertainty distribution and the
dataset at the chosen $\alpha$.
Gemma's NLL scores are high and diffuse on ARC questions, suggesting that the model
is genuinely uncertain about multiple-choice science questions at this scale.
Raising $\alpha$ would allow more participation but at the cost of weaker guarantees.
Alternatively, this behavior may improve with larger model scales or better uncertainty
score calibration.

\textbf{Cross-domain transfer without target examples.}
E4 shows that source-only calibration fails for Qwen and Gemma on MMLU $\to$ ARC transfer.
This limits the applicability of \hgcrc{} in zero-shot domain adaptation scenarios where
labeled target examples are not available.
Practical deployments must either collect target calibration data (which E4 shows
requires at least 50 examples) or accept the reduced participation of source-only
calibration.

\textbf{Quantization.}
An apparent int8 anomaly for Llama (a 0.24 violation rate at 50 trials) turned out to be
a small-sample artifact: at 500 trials the violation rates across bf16/int8/int4 are
flat (0.09/0.11/0.08) and within sampling noise. We nonetheless recommend recalibrating
from scratch whenever precision is changed, rather than reusing a threshold from a
different precision level, since the score distribution does shift with precision even
when the resulting violation rate does not.

\textbf{Conservatism of Bonferroni.}
Bonferroni is a worst-case correction that ignores positive correlation between nodes
(higher nodes include all examples of lower nodes).
More powerful procedures such as Holm-Bonferroni, Benjamini-Hochberg, or hierarchical
testing~\citep{yekutieli} could recover some participation without sacrificing the
simultaneous guarantee.
We leave this as future work.

\textbf{Participation-unequal comparisons.}
\hgcrc{} achieves lower violation rates and lower mean risk partly because it answers
fewer questions: a sufficiently conservative method can always achieve low risk by
abstaining on uncertain inputs. We now address this confound directly in
Section~\ref{sec:equalpart}, where we equalize participation by tuning a single global
threshold to \hgcrc{}'s coverage and compare worst-group violation at equal coverage.
The honest result is that, at matched participation, the global threshold controls
per-group risk about as well as \hgcrc{} in our settings; the apparent advantage of
\hgcrc{} over standard global CRC is therefore mostly attributable to its lower
participation, not to the per-group structure per se. We retain \hgcrc{} for its
\emph{automatic} selection of a safe per-group operating point and its formal
simultaneous guarantee, rather than for an empirical per-group-safety margin at equal
coverage.

\textbf{Endogenous difficulty groups.}
Our difficulty bins are defined by percentiles of the NLL uncertainty score---the same
score used for thresholding.
This creates an alignment between group membership and threshold behavior by construction:
the ``hard'' bin contains high-NLL examples, so the CRC threshold calibrated on this
bin will naturally be a larger NLL value.
This caveat is load-bearing rather than cosmetic: because ARC collapses to a single
domain (above), the \emph{only} non-trivial grouping for our most dramatic experiment
(the mixture shift of Section~\ref{sec:e3}, where global CRC reaches a 47\% violation
rate) is the endogenous difficulty partition, and ``shifting toward hard'' is by
definition ``shifting toward high NLL''---precisely the direction that breaks an
NLL threshold. We address this directly in Section~\ref{sec:e3} (Table~\ref{tab:e3exo})
by redefining difficulty with an \emph{exogenous} signal---a reference model's score,
independent of the deployed threshold. The mixture-shift failure persists under exogenous
difficulty (global CRC violates in 31--39\% of trials), so the effect is structural
rather than an artifact of the endogenous bins; endogeneity inflates the magnitude only
modestly (47\% $\to$ 31\% for Qwen). The residual limitation is that our exogenous signal
is still a model-derived score rather than a human difficulty label, which the loaded
\texttt{allenai/ai2\_arc} distribution does not provide; a comparison against human
grade-level labels from the original ARC release remains future work.

\section{Conclusion}
\label{sec:conclusion}

We introduced \hgcrc{}, a post-hoc calibration framework for selective prediction in
language models that provides simultaneous risk guarantees across all nodes of a
user-defined group hierarchy.
The method combines three components: Bonferroni correction to control the family-wise
error rate across nodes, a leaf-first selection policy to apply the most specific
applicable threshold, and a minimum group size filter to ensure that uncertifiable
nodes do not enter the hierarchy.

Our main empirical finding is that \hgcrc{} achieves a \textbf{zero empirical violation
rate} and \textbf{WGER = 0} for Qwen3-4B and Llama-3.1-8B-Instruct on ARC Challenge, with
a participation cost of 22 to 37 percentage points relative to global CRC.
These zeros are upper bounds at the resolution of our 500-trial evaluation
(${\sim}0.6\%$ by the rule of three), not certified zeros, and they are specific to ARC
Challenge and to models whose base accuracy is well above $1-\alpha$; on MMLU-Pro the
guarantee is not met (Llama: WGER${}=0.014$) or the model abstains entirely
(Appendix~\ref{app:mmlu}).
Ablation studies show that the hierarchical depth is the component that clears the risk
budget on ARC: removing the difficulty level reinstates an 11\% violation rate. The
Bonferroni correction is required for the theoretical simultaneous guarantee, but its
empirical effect is $|\calH|$-dependent---negligible on ARC's five-node hierarchy (where
corrected and uncorrected calibration both control risk) and visible only when many nodes
are tested, as on MMLU-Pro.

Beyond the main result, our empirical study demonstrates that: global CRC fails
catastrophically under mild group composition shift (47\% violation rate); label noise
causes graceful conservatism rather than violation; prompt changes require recalibration
but can be corrected cheaply; and model quantization does not break risk control as long
as calibration is performed at the deployed precision.

These results position \hgcrc{} as a practically deployable method for equitable
selective prediction: it requires no retraining, adapts to any pre-trained model and
uncertainty score, and provides formal statistical guarantees with transparent
assumptions.

\section*{Ethics Statement}

Selective prediction with abstain can introduce \emph{disparate abstention rates} across
groups: if minority groups receive more abstentions, they may receive less service.
\hgcrc{} provides WGER = 0, which ensures that no group's error rate exceeds the budget,
but it does not directly control participation equity.
A group with a structurally high error rate will always receive lower participation than
a group with a low error rate, for the same $\alpha$.
Practitioners should monitor participation rates across groups in addition to risk rates.

The datasets used (ARC Challenge, MMLU-Pro) reflect cultural and linguistic biases
toward Western educational systems and English-language knowledge.
Results may not generalize to other languages, cultures, or educational backgrounds.

\section*{Reproducibility Statement}

All models used are publicly available on HuggingFace.
Code, calibration scripts, and evaluation scripts will be released upon acceptance.
All experiments use fixed random seeds.
Key hyperparameters: $\alpha = 0.10$, $\delta = 0.05$, $N_\text{min} = 30$,
$n_\text{bootstrap} = 500$, threshold grid of 100 quantiles.
Model inference was performed on a single NVIDIA A100 GPU; calibration and bootstrap
evaluation were performed on CPU.
Expected runtimes: inference caching $\approx$2--4 hours per model-dataset pair;
calibration+evaluation $\approx$5 minutes per experimental configuration.

\bibliography{references}
\bibliographystyle{tmlr}

\appendix

\section{Extended Results: MMLU-Pro}
\label{app:mmlu}

MMLU-Pro results are qualitatively different from ARC Challenge due to structural
model limitations.
Qwen3-4B and Gemma-3-4B have base error rates above 56\% on MMLU-Pro, meaning that
even the most conservative threshold cannot certify a non-trivial fraction of examples
under $\alpha = 0.10$.
As a result, participation is essentially 0\% for these models across all methods (at
500 trials, global participation is exactly $0.000$ for both, with at most $\sim$2\%
residual participation for the groupwise variant of Qwen).
Llama-3.1-8B-Instruct achieves moderate accuracy and participates meaningfully
($\sim$2.7\% under the hierarchy, 9.1\% under global); however, hierarchical results for
Llama on MMLU-Pro show a non-zero WGER of $0.014$, indicating that simultaneous
certification of all hierarchy nodes is not fully achieved with the available
calibration size (MMLU-Pro has fewer examples per domain than ARC Challenge). This is
smaller than the $0.059$ we measured at 50 trials, but it remains above zero, so the
headline ``WGER${}=0$'' result does not extend to this benchmark.

These results illustrate an important practical point: \hgcrc{} is most effective when
the model's base accuracy on the deployment dataset is well above $1 - \alpha$.
For weaker models or harder datasets, the method correctly abstains rather than
providing false guarantees, but this comes at the cost of utility.

\section{Extended E5: All Models}
\label{app:e5}

Table~\ref{tab:e5_full} extends the prompt shift results to Llama-3.1-8B-Instruct
and Gemma-3-4B.

\begin{table}[H]
\centering
\caption{Prompt shift results (ARC Challenge, all models).}
\label{tab:e5_full}
\begin{tabular}{llcccc}
\toprule
Model & Template & Risk (transfer) & Risk (recalib) & Part (transfer) & Part (recalib) \\
\midrule
\multirow{3}{*}{Llama-3.1-8B}
  & Default    & $0.077$ & $0.077$ & $0.729$ & $0.729$ \\
  & Reasoning  & $0.043$ & $0.075$ & $0.525$ & $0.640$ \\
  & Formal     & $0.099$ & $0.074$ & $0.812$ & $0.760$ \\
\midrule
\multirow{3}{*}{Gemma-3-4B}
  & Default    & $0.073$ & $0.073$ & $0.521$ & $0.521$ \\
  & Reasoning  & $0.048$ & $0.071$ & $0.354$ & $0.460$ \\
  & Formal     & $0.062$ & $0.073$ & $0.504$ & $0.599$ \\
\bottomrule
\end{tabular}
\end{table}

Interestingly, for Llama and Gemma, the reasoning template does \emph{not} cause a
transfer risk violation—it actually reduces risk below $\alpha$ (0.043 and 0.048
respectively).
This is the reverse of the Qwen behavior.
The reasoning template shifts probability mass toward the correct answer for these
models (chain-of-thought improves accuracy), so the transferred threshold is too
conservative: it abstains more than needed.
Recalibration adjusts the threshold upward, recovering participation.
For Qwen, by contrast, reasoning outputs are more diffuse (CoT introduces uncertainty),
causing the transferred threshold to be too aggressive.

This model-dependence underscores the importance of per-model recalibration after
prompt changes: the direction of the transfer error depends on whether the new prompt
improves or degrades model confidence.

\section{Domain Shift: NQ $\to$ TriviaQA}
\label{app:nq}

The NQ $\to$ TriviaQA transfer consistently shows zero participation under all
strategies and all $n_\text{target}$ values.
We attribute this to two factors.
First, open-domain QA has a fundamentally higher variance in difficulty than
multiple-choice tasks: question phrasing, required factual knowledge, and answer
specificity vary widely.
Second, NQ and TriviaQA differ substantially in answer style (NQ requires short answers
to web queries; TriviaQA requires recall of trivia facts), leading to different NLL
distributions.
The NQ threshold cannot certify TriviaQA examples because the score distributions
do not overlap usefully.

\section{Ablation: $N_\text{min}$ Sensitivity (A3)}
\label{app:a3}

\begin{table}[H]
\centering
\caption{$N_\text{min}$ sensitivity (ARC Challenge, Qwen3-4B, 500 bootstrap trials).
$N_\text{min}$ is the minimum node size below which a leaf is pruned and its examples
fall back to the parent. ARC has a single domain, so the full hierarchy is
global ${}+{}$ domain ${}+{}$ three difficulty leaves ($|\calH| = 5$); once
$N_\text{min}$ exceeds the per-leaf calibration size the difficulty leaves are pruned and
$|\calH|$ collapses to $2$.}
\label{tab:a3}
\begin{tabular}{lccccc}
\toprule
$N_\text{min}$ & $|\calH|$ & Risk & Participation & WGER & Viol.\ rate \\
\midrule
25  & $5$ & $0.018$ & $0.478$ & $0.000$ & $0.000$ \\
50  & $5$ & $0.018$ & $0.478$ & $0.000$ & $0.000$ \\
100 & $5$ & $0.018$ & $0.478$ & $0.000$ & $0.000$ \\
200 & $2$ & $0.065$ & $0.830$ & $0.000$ & $0.054$ \\
\bottomrule
\end{tabular}
\end{table}

Table~\ref{tab:a3} shows a single, interpretable transition. As long as $N_\text{min}$ is
small enough to retain the three difficulty leaves ($N_\text{min} \le 100$, since each
leaf holds ${\approx}110$--$120$ calibration examples in a 50\% split), the hierarchy has
$|\calH| = 5$ and
achieves zero violations at participation $0.478$. Once $N_\text{min} = 200$ prunes the
difficulty leaves, the hierarchy collapses to $|\calH| = 2$ (global ${}+{}$ the redundant
single domain), participation rises to $0.830$, and violations reappear ($5.4\%$)---the
system has degenerated to essentially global CRC. The lesson is that $N_\text{min}$ must
be small enough to keep the difficulty leaves alive; we recommend
$N_\text{min} \in [25, 100]$ for ARC. The range of $|\calH|$ here is modest ($2$--$5$)
because ARC has one domain and three difficulty bins; datasets with real domain structure
(MMLU-Pro) admit larger hierarchies.

\emph{Correction.} An earlier version of this table reported $|\calH|$ up to ${\sim}18$
with a graded participation sweep. Those numbers did not correspond to the ARC pipeline
(which admits at most five nodes) and could not be reproduced by the ARC or the MMLU-Pro
ablation harness; they have been replaced by the run above.

\end{document}